\documentclass[letterpaper,twocolumn,10pt]{article}
\usepackage{usenix-2020-09}
\usepackage{threeparttable}
\usepackage{algorithm,algorithmicx,algpseudocode}
\usepackage{xspace}
\usepackage{tablefootnote}
\usepackage{amsthm}
\usepackage{graphicx}
\usepackage{textcomp}
\usepackage{xcolor}
\usepackage{booktabs}
\usepackage{multirow}
\usepackage{diagbox}
\usepackage{amsmath}
\usepackage{hyperref}
\newtheorem{theorem}{Theorem}
\newtheorem{definition}{Definition}

\newtheorem{proposition}{Proposition}
\usepackage{caption}

\newcommand{\framework}{\textsc{PatchCURE}\xspace}

\usepackage{amssymb}
\usepackage{pifont}
\newcommand{\cmark}{\ding{51}}%
\newcommand{\xmark}{\ding{55}}%

\newcommand{\st}{\mathrm{~~s.t.~~}}
\newcommand{\splitk}{k}

\newcommand{\srf}{\text{srf}}
\newcommand{\lrf}{\text{lrf}}
\newcommand{\bM}{{\mathbb M}}

\newcommand{\bC}{{\mathbb C}}

\newcommand{\cA}{{\mathcal A}}

\newcommand{\cF}{{\mathcal F}}

\newcommand{\cM}{{\mathcal M}}

\newcommand{\cR}{{\mathcal R}}

\newcommand{\cX}{{\mathcal X}}
\newcommand{\cY}{{\mathcal Y}}

\newcommand{\bff}{\mathbf{f}}

\newcommand{\bfm}{\mathbf{m}}

\newcommand{\bfr}{\mathbf{r}}

\newcommand{\bfx}{\mathbf{x}}

\begin{document}

\pagestyle{plain}
\title{\framework: Improving Certifiable Robustness, Model Utility, and Computation Efficiency of Adversarial Patch Defenses}
\author{\rm{Chong Xiang$^1$, Tong Wu$^1$, Sihui Dai$^1$, Jonathan Petit$^2$, Suman Jana$^3$, Prateek Mittal$^1$}\\$^1$Princeton University, $^2$Qualcomm Technologies, Inc., $^3$Columbia University}
\maketitle

\begin{abstract}
    State-of-the-art defenses against adversarial patch attacks can now achieve strong certifiable robustness with a marginal drop in model utility. However, this impressive performance typically comes at the cost of 10-100$\times$ more inference-time computation compared to undefended models -- the research community has witnessed an intense three-way trade-off between certifiable robustness, model utility, and computation efficiency. In this paper, we propose a defense framework named \framework to approach this trade-off problem. \framework provides sufficient ``knobs'' for tuning defense performance and allows us to build a family of defenses: the most robust \framework instance can match the performance of any existing state-of-the-art defense (without efficiency considerations); the most efficient \framework instance has similar inference efficiency as undefended models. Notably, \framework achieves state-of-the-art robustness and utility performance across all different efficiency levels, e.g., 16-23\% absolute clean accuracy and certified robust accuracy advantages over prior defenses when requiring computation efficiency to be close to undefended models. The family of \framework defenses enables us to flexibly choose appropriate defenses to satisfy given computation and/or utility constraints in practice.\footnote{Our source code is available at \url{https://github.com/inspire-group/PatchCURE}.}
    
\end{abstract}

\begin{figure*}[t]
    \centering
    \includegraphics[width=\linewidth]{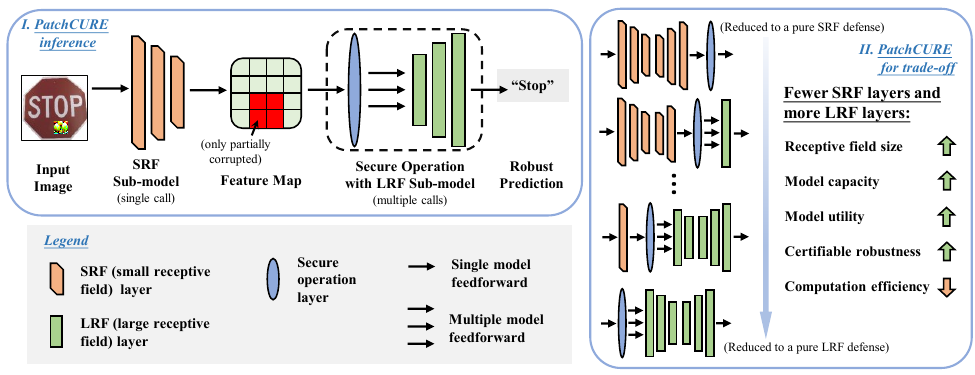}
    \caption{\textbf{\framework overview.} \textbf{I. \framework inference (Section~\ref{sec-defense-alg}):} Given an input image, we first call an SRF (small receptive field) sub-model \textit{once} to extract an intermediate feature map. The use of SRF ensures that only part of the features is corrupted. Next, we leverage secure operation, which typically involves \textit{multiple} calls to an LRF (large receptive field) sub-model, for final predictions. \textbf{II. \framework for the trade-off problem (Section~\ref{sec-defense-discussion}).} We can adjust the combination of SRF and LRF layers to balance the three-way trade-off. As we use fewer SRF layers and more LRF layers, the defense model (with a fixed number of total layers) normally has larger receptive fields, larger model capacity, better model utility, and higher certifiable robustness, but poorer computation efficiency.}
    \label{fig-overview}
\end{figure*}

\begin{figure}[t]
    \centering
    \includegraphics[width=0.95\linewidth]{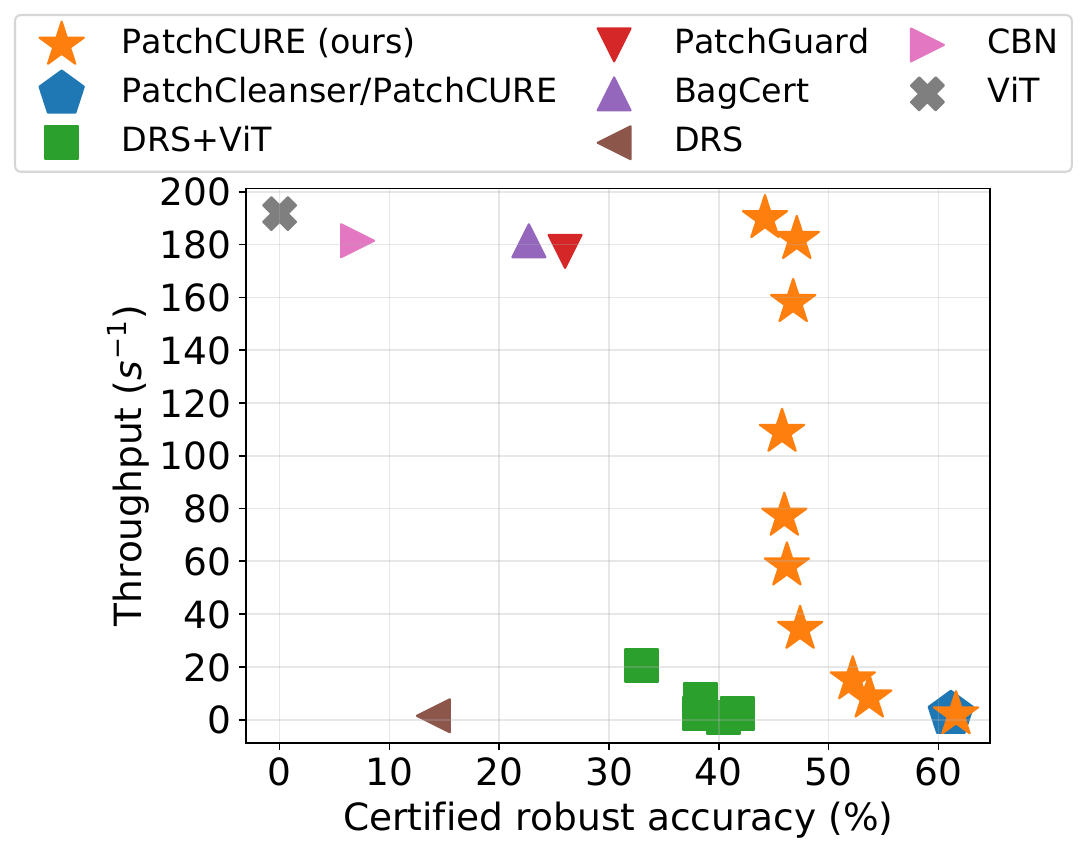}
    \caption{Certified robust accuracy and inference throughput (img/s) for different defenses on ImageNet-1k~\cite{imagenet}: (i) our \framework instances with different settings; (ii) PatchCleanser~\cite{xiang2022patchcleanser} -- also a special instance of \framework; (iii) DRS+ViT, including Smoothed ViT~\cite{salman2022certified}, ECViT~\cite{ecvit}, and ViP~\cite{li2022vip}; (iv) PatchGuard~\cite{xiang2021patchguard}; (v) BagCert~\cite{bagcert}; (vi) De-Randomized Smoothing (DRS)~\cite{levine2020randomized}; (vii) Clipped BagNet (CBN)~\cite{zhang2020clipped}; (viii) undefended ViT~\cite{vit}. Certified robustness considers one 2\%-pixel square patch anywhere on the image.}
    \label{fig-main-comparison}
\end{figure}
\section{Introduction}\label{sec-intro}

The adversarial patch attack~\cite{brown2017adversarial} against computer vision models overlays a malicious pixel patch onto an image to induce incorrect model predictions. Notably, this attack can be realized in the physical world by printing and attaching the patch to real-world objects: images/photos captured from that physical scene would become malicious. The physically realizable nature of patch attacks raises significant concerns in security and safety-critical applications like autonomous vehicles~\cite{rp2}, face authentication~\cite{wei2022adversarial}, security surveillance~\cite{xu2020adversarial}, and thus, motivates the design of defense mechanisms.

Among all the efforts in mitigating adversarial patch attacks, research on certifiably robust defenses stands out with remarkable progress~\cite{chiang2020certified,levine2020randomized,xiang2021patchguard,salman2022certified,ecvit,xiang2022patchcleanser,li2022vip}. These defenses aim to provide provable/certifiable robustness guarantees that hold for any attack strategy within a given threat model, including attackers with full knowledge of the defense algorithms and model details. The property of certifiable robustness provides a pathway toward ending the arms race between attackers and defenders. 
Notably, the state-of-the-art certifiably robust defenses~\cite{xiang2022patchcleanser} now can achieve high certifiable robustness while only marginally affecting the model utility (e.g., 1\% accuracy drop on ImageNet~\cite{imagenet}).

However, the impressive robustness and utility performance comes at the cost of overwhelming computation overheads. Many defenses~\cite{xiang2022patchcleanser,salman2022certified,ecvit,li2022vip} require 10-100$\times$ more inference-time computation compared to undefended models, and this makes them computationally prohibitive for real-world deployment. On the other hand, there exist efficient defenses~\cite{zhang2020clipped,xiang2021patchguard,bagcert} with small computation overheads, but they all suffer from a significant drop in certifiable robustness and model utility (e.g., 30+\% absolute drop from the state-of-the-art). \textit{The research community has witnessed an intense {three-way} trade-off between certifiable robustness, model utility, and computation efficiency~\cite{survey}}. 

Unfortunately, there is no existing method to systematically study this three-way trade-off problem. State-of-the-art (inefficient) defenses, e.g., PatchCleanser~\cite{xiang2022patchcleanser}, do not have a design point to achieve similar computation efficiency as undefended models. Efficient defenses, e.g., PatchGuard~\cite{xiang2021patchguard}, lack approaches to trade part of their efficiency for better utility and robustness performance.
In this paper, we make a \textit{first} attempt to address these challenges: we propose a defense framework named \framework to consolidate \underline{\textbf{C}}ertifiable \underline{\textbf{R}}obustness, Model \underline{\textbf{U}}tility, and Computation \underline{\textbf{E}}fficiency.

\textbf{Contributions.} The main contribution of \framework is identifying a key factor of the three-way trade-off problem: the model \textit{receptive field}, i.e., the image region each extracted feature is looking at. We observe that most existing defenses can be categorized as either small receptive field (SRF) defenses or large receptive field (LRF) defenses. SRF defenses~\cite{xiang2021patchguard,levine2020randomized,salman2022certified,ecvit} use a model with SRF for feature extraction so that there is only a limited number of corrupted features that marginally interfere with the prediction. In contrast, LRF defenses~\cite{xiang2022patchcleanser,mccoyd2020minority} aim to directly mask out the entire patch from the input images and then use a high-performance model with LRF for final predictions. Interestingly, SRF techniques criticize LRF defenses for excessive computation overheads (LRF techniques require multiple model predictions on different masked images) while LRF techniques criticize SRF defenses for poor model utility (SRF limits the information received by each feature and hurts model capacity). In \framework, we propose a \textit{unified} generalization of SRF and LRF techniques.

\textbf{\framework: a framework that unifies SRF and LRF techniques.} We provide an overview of \framework in Figure~\ref{fig-overview} (see the figure caption for more details). A \framework defense has three modules: an SRF sub-model, an LRF sub-model, and a secure operation ``layer'' (an abstract layer representing a robust prediction algorithm/procedure, e.g., the double-masking procedure from PatchCleanser~\cite{xiang2022patchcleanser}). To make a prediction on an input image (upper left of the figure), \framework first uses the SRF sub-model to extract intermediate features and then leverages the secure operation layer, together with multiple calls to the LRF sub-model, for a final robust prediction. To balance the three-way trade-off (right of the figure), \framework adjusts the portion/combination of SRF and LRF layers within the end-to-end defense model.

\textbf{State-of-the-art efficient defense instances.}
Using our \framework framework, we can build state-of-the-art efficient defenses (with similar inference speed as undefended models). In doing so, we first design a ViT-SRF architecture (a ViT~\cite{vit} variant with SRF) to instantiate the SRF sub-model. Next, to minimize computation overheads, we instantiate a lightweight LRF sub-model, which is composed of a linear feature aggregation followed by a classification head. Finally, we use the double-masking algorithm~\cite{xiang2022patchcleanser} as the secure operation. In Figure~\ref{fig-main-comparison}, we demonstrate that the most efficient instances of \framework (top stars) have similar inference speed as undefended ViT (cross), while significantly outperforming prior efficient defenses (top triangles) in terms of certifiable robustness (more than 18\% absolute improvement). Moreover, these efficient \framework defense instances even outperform all but one existing inefficient defense (bottom squares/triangles).\footnote{For simplicity, we did not plot the model utility performance in Figure~\ref{fig-main-comparison}; we will demonstrate \framework's superiority in model utility in Section~\ref{sec-eval}.}

\textbf{A systematic way to balance the three-way trade-off.} 
Furthermore, we demonstrate that the \framework framework provides sufficient knobs to balance the three-way trade-off. We systematically explore different defense parameters and plot a family of \framework defenses in Figure~\ref{fig-main-comparison}. As shown in the figure, \framework instances (stars) can easily bridge the robustness gap between the most efficient \framework (top stars) and the state-of-the-art PatchCleanser~\cite{xiang2022patchcleanser} (pentagon; also a special instance of \framework as discussed in Section~\ref{sec-defense}). Moreover, we note that \framework also achieves the best robustness performance (and utility performance) across all different efficiency levels.
With \framework, we can flexibly build the optimal defense that satisfies certain computation efficiency or model utility requirements.

We summarize our contributions as follows.
\begin{enumerate}
    \item We propose a \framework defense framework that unifies disparate SRF and LRF techniques for approaching the three-way trade-off problem.
    \item We designed a ViT-SRF architecture to instantiate efficient \framework defenses and achieve 18+\% absolute robustness improvements from prior efficient defenses.
    \item We experimentally demonstrate that \framework provides sufficient knobs to balance robustness, utility, and efficiency, and also achieves state-of-the-art robustness and utility performance across different efficiency levels.
\end{enumerate}

\section{Preliminaries}\label{sec-prelim}
In this section, we formulate the research problem, discuss the important concept of model receptive fields, and present an overview of existing SRF and LRF defense techniques. 

\subsection{Problem Formulation}\label{sec-formulation}
In this subsection, we detail image classification models, adversarial patch attacks, certifiable robustness, and three performance dimensions studied in the trade-off problem. We summarize important notations in Table~\ref{tab-notation}.

\textbf{Image classification models.} In this paper, we study image classification models. We let $\mathcal{X}\subset [0,1]^{H\times W\times C}$ denote the input image space: $H,W,C$ correspond to the height, width, and number of channels of the image; $[0,1]$ is the range of normalized pixel values. We next denote the label space as $\mathcal{Y}=\{0,1,\cdots,N-1\}$, where $N$ is the total number of classes. Finally, we denote an image classification model as $\mathbb{M}:\mathcal{X}\rightarrow \mathcal{Y}$, which maps an input image $\mathbf{x}\in \mathcal{X}$ to its prediction label $y\in\mathcal{Y}$. 

We further use $\cF\in\mathbb{R}^{H^\prime\times W^\prime \times C^\prime}$ to denote the space of the intermediate feature map. We will overload the notation $\mathbb{M}$ as (1) a feature extractor that maps an input image to intermediate features $\bM_0:\cX\rightarrow \cF$, or (2) a classifier that makes predictions based on the intermediate feature map $\bM_1:\cF\rightarrow\cY$.

\begin{table}[t]
    \centering
    \caption{Summary of important notations}\label{tab-notation}
 \resizebox{\linewidth}{!}
  { \begin{tabular}{l|l|l|l}
    \toprule
    \textbf{Notation} & \textbf{Description} & \textbf{Notation} & \textbf{Description} 
    \\
    \midrule
    $\mathbb{M}$ &Classification model& $\mathbf{x}\in\mathcal{X}$ & Input image\\
 $\bff\in\mathcal{F}$ &  Intermediate feature &  $y\in\mathcal{Y}$ &  Class label  \\
     $\mathbf{r}\in\mathcal{R}$ & Patch region &  $\cA_\mathcal{R}$ & Threat model \\
     $k$ & Model splitting index & $L$ & \#Layers\\
    $\mathbf{rf}$ & RF size & $\mathcal{M}$ & Mask set\\
      \bottomrule
    \end{tabular}}
\end{table}

\textbf{Adversarial patch attack.} The adversarial patch attack~\cite{brown2017adversarial} is a type of test-time evasion attack. Given a model $\mathbb{M}$, an image $\mathbf{x}$, and its correct label $y$, the attacker aims to generate an adversarial image $\mathbf{x}^\prime \in \mathcal{A}(\mathbf{x}) \subset \mathcal{X}$ satisfying certain constraints $\mathcal{A}$ to induce incorrect predictions $\mathbb{M}(\mathbf{x}^\prime) \neq y$.
The patch attack constraint $\mathcal{A}$ allows the attacker to introduce \textit{arbitrary} pixel patterns within \textit{one} {restricted} image region (i.e., the patch region); the patch region can be at \textit{any location} chosen by the attacker. This constraint threat model is widely used in prior certifiably robust defenses~\cite{zhang2020clipped,levine2020randomized,bagcert,salman2022certified,ecvit,li2022vip,xiang2021patchguard,xiang2022patchcleanser}.\footnote{We quantitatively discuss defenses for multiple patches in Appendix~\ref{apx-pc-exp}.}

Formally, we use a binary pixel tensor $\mathbf{r} \in \{0,1\}^{H\times W}$ to represent each restricted patch region: pixels within the region are set to zeros, and others are ones. We next let $\cR$ denote a set of all possible patch regions $\bfr$ (e.g., patches at all different image locations). Then, we can formalize the constraint as $\mathcal{A}_{\mathcal{R}}(\mathbf{x})$ as $\{\mathbf{r}\odot \mathbf{x} + (\mathbf{1}-\mathbf{r}) \odot \mathbf{x}^{\prime} \ |\  \mathbf{x},\mathbf{x}^\prime \in \mathcal{X}, \mathbf{r} \in \mathcal{R}\}$, where $\odot$ is the element-wise multiplication operator, and $\bfx^\prime$ contains malicious pixels. When clear from the context, we drop $\mathcal{R}$ and simplify $\mathcal{A}_{\mathcal{R}}$ as $\mathcal{A}$.

\textbf{Certifiable robustness.} We study defense algorithms whose robustness can be formally proved or certified. That is, given an image $\bfx$, its correct label $y$, and a patch attack constraint $\cA_\cR$, we aim to build a defense model $\bM$ such that we can certify that 
\begin{equation}\label{eqn-cert}
    \forall \bfx^\prime \in \cA_\cR(\bfx), \mathbb{M}(\mathbf{x}^\prime)=\mathbb{M}(\mathbf{x})=y
\end{equation}
In addition to the defense model $\bM$, we will also develop a special certification procedure $\bC:\cX\times\cY\times\mathbb{P}(\cA_\cR)\rightarrow\{\texttt{True},\texttt{False}\}$ to determine whether Equation~\ref{eqn-cert} holds for an image $\bfx$, its label $y$, and threat model $\cA_\cR$. Note that the universal quantifier $\forall$ requires the certification procedure $\mathbb{C}$ to account for all possible attackers within the threat model $\mathcal{A}_\cR$, who could have full knowledge of the defense algorithm and setup. This certification ensures that any robustness we claim will not be compromised by adaptive attackers, which is a significant advantage over empirical defenses~\cite{hayes2018visible,naseer2019local,wu2019defending,rao2020adversarial,Mu2021defending,cosgrove2020robustness} without formal robustness guarantees.

\textbf{Three performance dimensions.} We will study the trade-off problem between three performance dimensions. The first is \textit{certifiable robustness}, as discussed above. We will run the certification procedure over all images of a labeled test dataset and report certified robust accuracy as the fraction of images for which the procedure $\mathbb{C}$ returns \texttt{True}. The second is \textit{model utility} -- model accuracy on clean images without adversarial patches, also termed as clean accuracy.\footnote{Model utility can be viewed as an upper bound of achievable robustness since we do not expect robust accuracy to be higher than clean accuracy. Therefore, we can sometimes find utility and robustness tied together when discussing the three-way trade-off.} 
The third is \textit{computation efficiency}; we measure it via inference throughput -- the number of images a model can process within every second. 

\begin{figure}
    \centering
    \includegraphics[width=0.9\linewidth]{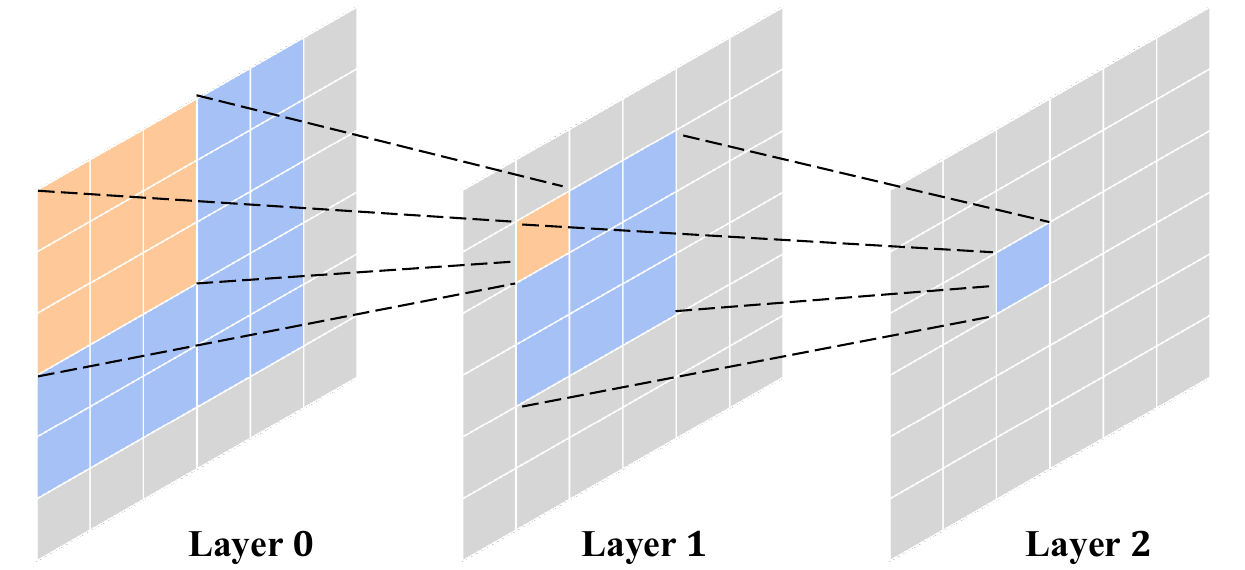}
    \caption{\textbf{Illustration of model receptive field. }For a convolutional network with a kernel size of $3$ and stride size of $1$, the blue cell in Layer $2$ is affected by $3\times3$ cells in Layer $1$ and $5\times5$ cells in Layer $0$ (the model input).}
    \label{fig-rf}
\end{figure}

\subsection{Receptive Fields of Vision Models}

The receptive field of a vision model is the input image region that each extracted feature is looking at, or is affected by; we provide a visual example in Figure~\ref{fig-rf}. A conventional vision model first extracts features from different image regions (with different focuses), and then aggregates all extracted features and makes an informed prediction. 

There has been a line of research studying the effect of receptive field size on model performance~\cite{bagnet,araujo2019computing,le2017receptive,luo2016understanding}. Normally, larger receptive fields lead to larger model capacity and thus better model performance (as long as the model is well-trained with enough data). For example, a convolutional network like ResNet~\cite{resnet} usually has better performance when it has deeper layers with larger receptive fields; the emerging powerful Vision Transformer (ViT)~\cite{vit} architecture allows features in every layer to have a large/global receptive field of the entire image. 
In this paper, we will demonstrate that the receptive field size also plays an important role in the three-way trade-off problem for certifiably robust patch defenses. 

\subsection{Overview of SRF and LRF Defenses}\label{sec-prelim-pgpc}

In this subsection, we provide an overview of existing small receptive field (SRF) and large receptive field (LRF) defense techniques. 

\textbf{SRF defenses.} The use of SRF was first explicitly discussed in PatchGuard~\cite{xiang2021patchguard}. Its key insight is that: using models with SRF for feature extraction can limit the number of features that see (and hence, are affected by) the adversarial patch. The maximum number of corrupted features $p^f$ can be computed as 
\begin{equation}\label{eqn-corruption}
    p^f = \lceil(p+r-1)/s\rceil
\end{equation}
where $p$ is the patch size in the image space, $r$ is the receptive field size of the SRF model, $s$ is the effective stride of the model receptive field (i.e., the distance between the centers of receptive fields of two adjacent features). The correctness of this formula was proven in~\cite{xiang2021patchguard}. Next, an SRF defense performs lightweight secure feature aggregation (e.g., clipping, majority voting) on the partially corrupted feature map for robust predictions. The idea of SRF has been shown effective and adopted by many defenses~\cite{zhang2020clipped,levine2020randomized,bagcert,salman2022certified,ecvit,li2022vip}. 

\textit{Strength: high efficiency.} Since the secure aggregation operates on the final feature map, its computation complexity can be as low as a linear transformation layer. Therefore, the computation of SRF defenses is dominated by the feedforward pass of the SRF model, which can be made as efficient as standard CNN or ViT (see Section~\ref{sec-srf} for more details). 

\textit{Weakness: poor utility.} However, though the use of SRF bounds the number of corrupted features, it also limits the information received by each feature. As a result, the model utility is affected. For example, the clean accuracy on ImageNet-1k~\cite{imagenet} reported in the original PatchGuard paper~\cite{xiang2021patchguard} is only 54.6\% while standard ResNet-50~\cite{resnet} and ViT-B~\cite{vit} can achieve 80+\% accuracy without additional training data~\cite{strikeback,he2022masked}. 

\textbf{LRF defenses.} The key idea of LRF defenses is to remove the patch from the input image and then use a high-performance LRF model to recover robust predictions. Its most representative defense, PatchCleanser~\cite{xiang2022patchcleanser}, proposed a \textit{double-masking} algorithm that applies different pixel masks to the input image and analyzes model predictions (with LRF) on different masked images to recover the correct prediction. The intuition behind the masking defense is that: model predictions on images with different masks usually have a unanimous agreement on clean image (predictions are robust to partial occlusions), but disagree when there is an adversarial patch (when the patch is completely masked, the model prediction changes to benign). We additionally provide pseudocode for the PatchCleanser~\cite{xiang2022patchcleanser} in Appendix~\ref{apx-pc}.

\textit{Strength: high utility\&robustness.} Since LRF defenses operate on the input image, they are compatible with any high-performance image classifiers (which usually have LRF). This allows LRF defenses like PatchCleanser~\cite{xiang2022patchcleanser} to maintain a very high model utility (e.g., 1\% drops on ImageNet-1k from vanilla ResNet and ViT) while achieving state-of-the-art certifiable robustness.

\textit{Weakness: low computation efficiency.} The downside of image-space operations of LRF defenses is that it normally requires performing model feedforward \textit{multiple times} on different masked images. As a result, LRF defenses can easily incur 10+ times more computation compared to undefended models, making it impractical for real-world deployment.

In summary, SRF and LRF techniques are widely viewed as two distinct approaches to building defenses with different strengths and weaknesses~\cite{survey}. In the next section, we will discuss how \framework unifies these disparate SRF and LRF techniques to approach the three-way trade-off problem.

\section{\framework Framework}\label{sec-defense}
In this section, we discuss our \framework design. We start with our defense insights and the full \framework algorithm. We then elaborate on approaches for building SRF models, discuss robustness certification,  and conclude with \framework's instantiation strategy.

\begin{table}[t]
    \centering
    \caption{Comparison for the SRF defense, the LRF defense, and \framework}
    \label{tab-defense-comparison}
     \resizebox{\linewidth}{!}
    {\begin{tabular}{l|l|l|l}
    \toprule
    Technique   & Secure Operation Loc.  & Utility\&Robustness & Efficiency  \\
         \midrule
         SRF (e.g.,~\cite{xiang2021patchguard}) &  Final (feature) layer  &Poor/Fair

         & Good \\
         LRF (e.g.,~\cite{xiang2022patchcleanser}) &  Input (image) layer  & Good & Poor \\
         \framework (ours) & Flexible & Tunable & Tunable\\
        \bottomrule
    \end{tabular}}
\end{table}

\subsection{\framework Insights}

Section~\ref{sec-prelim-pgpc} demonstrated a tension between certifiable robustness, model utility, and computation efficiency: existing defenses with SRF or LRF techniques struggle to perform well in all three dimensions. In Table~\ref{tab-defense-comparison}, we summarize and compare different properties of SRF and LRF techniques. This table helps us identify a key factor of the three-way trade-off problem, which inspires our \framework design. 

\textbf{Key factors in the trade-off problem.} From Table~\ref{tab-defense-comparison}, we find that different \textit{receptive field sizes} lead to different secure operation locations (i.e., where the defense logic is applied to) and eventually different defense properties. SRF defenses (e.g., PatchGuard~\cite{xiang2021patchguard}) can operate on the final feature layer. This design only requires SRF defenses to perform \textit{one} expensive model feedforward; that is, their secure operation can reuse the extracted feature map to achieve high computation efficiency. However, the use of SRF hurts model utility and robustness. On the other hand, LRF defenses (e.g., PatchCleanser~\cite{xiang2022patchcleanser}) operate on the input image. This makes them compatible with high-performance LRF models to achieve high model utility (and robustness). However, image-space defenses suffer from low efficiency because they need to perform \textit{multiple} expensive end-to-end model feedforward passes on different modified images (e.g., masked images).

\textbf{\framework as a unified defense.} To benefit from the strengths of SRF and LRF defenses, we propose a unified \framework defense that leverages both SRF and LRF techniques (bottom row of Table~\ref{tab-defense-comparison}). Recall our defense overview in Figure~\ref{fig-overview}. A \framework defense has three modules: an SRF sub-model, an LRF sub-model, and a secure operation layer/procedure that leverages the LRF sub-model. At the inference time, we call the SRF sub-model \textit{once} to extract intermediate features and then activate the secure operation with \textit{multiple} calls to the LRF sub-models to remove corrupted features for a robust final prediction. 
Notably, the SRF component allows us to reuse the extracted intermediate feature map to \textbf{save computation}. In the meanwhile, the use of SRF also ensures that only a limited number of features are corrupted, so we can still use secure operation techniques like double-masking algorithm~\cite{xiang2022patchcleanser} to remove corrupted features for \textbf{certifiable robustness} (note that all intermediate features are likely to be corrupted without the use of SRF~\cite{xiang2021patchguard}). 
Furthermore, the use of LRF makes the end-to-end model have large receptive fields so that we can retain \textbf{high model utility}.

\subsection{\framework Algorithm}\label{sec-defense-alg}
In this subsection, we discuss details of \framework algorithm, which includes model construction, i.e., how to build SRF and LRF sub-models, model inference, and robustness certification procedures. We provide their pseudocode in Algorithm~\ref{alg-all}.

\textbf{Model construction (Lines~\ref{ln-construct-s}-\ref{ln-construct-e}).} We build SRF and LRF sub-models based on existing state-of-the-art models, which we call \textit{base models}. First, we select an off-the-shelf base model architecture $\bM_\text{b}$ (e.g., ViT~\cite{vit}) and pick its $\splitk^\text{th}$ backbone layer (e.g., the second self-attention layer in ViT) as the \textit{splitting layer}. Second, we split the model at this layer into two sub-models using the $\textsc{Split}(\cdot)$ procedure. We have $\bM_{0},\bM_{1}=\textsc{Split}(\bM_\text{b},k)\st \bM(\bfx) = \bM_1(\bM_0(\bfx))$; $\bM_{0}$ contains layers with indices from $0$ to $\splitk-1$ while $\bM_{1}$ has the remaining layers. Third, we keep the second sub-model $\bM_1$ unchanged as the LRF sub-model $\bM_\text{lrf}=\bM_1$ (note that vanilla models normally have LRF for high utility) and convert the first sub-model $\bM_{0}$ into an SRF sub-model $\bM_\text{srf} = \textsc{ToSRF}(\bM_0,\mathbf{rf})$ with a receptive field size of $\mathbf{rf}$. We will discuss the details of $\textsc{ToSRF}(\cdot)$ in Section~\ref{sec-srf} -- how to build attention-based (e.g., ViT) and convolution-based (e.g., ResNet) SRF models with high computation efficiency. Finally, we \textit{conceptually} insert a secure operation ``layer" between the SRF sub-model and the LRF sub-model (recall Figure~\ref{fig-overview}). The secure operation layer represents a robust prediction algorithm/procedure $\textsc{SO}(\cdot)$; its design choices include the double-masking algorithm proposed in PatchCleanser~\cite{xiang2022patchcleanser} and the Minority Reports algorithm~\cite{mccoyd2020minority}. We will focus on the double-masking algorithm in this paper since it is the \textit{state-of-the-art} certifiably robust algorithm to \textit{recover} correct predictions. We provide details of the double-masking algorithm in Appendix~\ref{apx-pc}. We also discuss alternative secure operation choices in Section~\ref{sec-discussion}.

\begin{algorithm}[t]
    \centering
    \caption{\framework algorithm}\label{alg-all}
    \begin{algorithmic}[1]
    \Require Base model $\bM_b$, splitting layer index $k$, SRF size $\mathbf{rf}$, secure operation algorithm $\textsc{SO}(\cdot)$ and its certification procedure $\textsc{SO-Cert}(\cdot)$, secure operation parameters $\cM$, patch threat model $\cA_\cR$
    \Procedure{PCURE-Construct}{$\bM_b,k,\mathbf{rf}$} \label{ln-construct-s}
    \State $\bM_0,\bM_1\gets\textsc{Split}(\bM_b,k)$ \Comment{Split model at $k^\text{th}$ layer}
    \State $\bM_{\srf}\gets\textsc{ToSRF}(\bM_0,\mathbf{rf})$ \Comment{Convert $\bM_0$ to SRF}
    \State $\bM_\lrf\gets\bM_1$ \Comment{Keep $\bM_1$ as LRF}
    \State \Return $\bM_{\srf},\bM_\lrf$
    \EndProcedure\label{ln-construct-e}
    \item[]
    \Procedure{PCURE-Infer}{$\bfx,\bM_{\srf},\bM_\lrf,\cM$}\label{ln-infer-s}
    \State $\bff\gets\bM_\srf(\bfx)$ \Comment{Extract SRF features}
    \State $\hat{y}\gets\textsc{SO}(\bff,\bM_\lrf,\cM)$ \Comment{Secure operation on $\bff$}
    \State\Return $\hat{y}$
    \EndProcedure\label{ln-infer-e}
    \item[]
    \Procedure{PCURE-Certify}{$\bfx,y,\bM_{\srf},\bM_\lrf,\cM,\cA_\cR$}\label{ln-certify-s}
    \State $\bff\gets\bM_\srf(\bfx)$\Comment{Extract SRF features}
    \State $\cA_\cR^f\gets\textsc{Map}(\bM_\srf,\cA_\cR)$ \Comment{To feature-space adversary}\label{ln-map}
    \State $c\gets\textsc{SO-Cert}\label{ln-SO-Cert}(\bff,\bM_\lrf,\cM,\cA_\cR^f,y)$  \Comment{Certification}
    \State\Return $c$
    \EndProcedure\label{ln-certify-e}
    
\end{algorithmic}
\end{algorithm}

\textbf{Model inference (Lines~\ref{ln-infer-s}-\ref{ln-infer-e}).} Once we build the SRF and LRF sub-models and decide on the secure operation algorithm, the inference is straightforward. Given an input image $\bfx$, we first use the SRF model $\bM_\srf(\bfx)$ to extract features $\bff=\bM_\srf(\bfx)$. Next, we activate the secure operation $\textsc{SO}(\cdot)$ on the feature tensor for the final prediction: $\hat{y}=\textsc{SO}(\bff,\bM_\lrf,\cM)$. The procedure $\textsc{SO}(\cdot)$ takes as inputs a feature/image tensor $\bff$, an LRF model $\bM_\lrf$, and any parameters $\cM$ of the secure operation algorithm, and outputs robust prediction labels $\hat{y}\in\cY$. 

\textbf{Robustness certification (Lines~\ref{ln-certify-s}-\ref{ln-certify-e}).} Given a \textit{clean} image $\bfx$, we first extract features using the SRF model $\bff=\bM_\srf(\bfx)$. Next, we map the image-space adversary to the feature-space adversary: the procedure $\textsc{Map}(\cdot)$ uses  Equation~\ref{eqn-corruption} to calculate the number of corrupted features $p^f$ based on the adversarial patch size $p$ in the input image, convert each image-space patch region $\bfr$ to feature-space corruption region $\bfr^f$ accordingly, and then generate the feature-space adversary threat model $\cA_\cR^f$ for further robustness analysis.
Finally, we call the robustness certification procedure of the secure operation algorithm on the feature tensor (with the feature-space threat model): $\hat{y}=\textsc{SO-Cert}(\bff,\bM_\lrf,\cM,\cA_\cR^f,y)$. The certification procedure $\textsc{SO-Cert}(\cdot)$ takes as inputs a feature/image tensor $\bff$, the LRF model $\bM_\lrf$, the secure operation parameters $\cM$ used for the model inference, the feature-space threat model $\cA_\cR^f$, and the ground-truth label $y\in\cY$, and outputs a boolean variable $b\in\{\texttt{True},\texttt{False}\}$ indicating if the certification succeeds, i.e., $\hat{y}=\bM_\lrf(\bff^\prime)= y, \forall \bff^\prime\in\cA^f_\cR(\bff)$.  We will state and prove the correctness of our certification procedure $\textsc{PCURE-Certify}(\cdot)$ in Section~\ref{sec-cert}.

\begin{figure}[t]
    \centering
    \includegraphics[width=\linewidth]{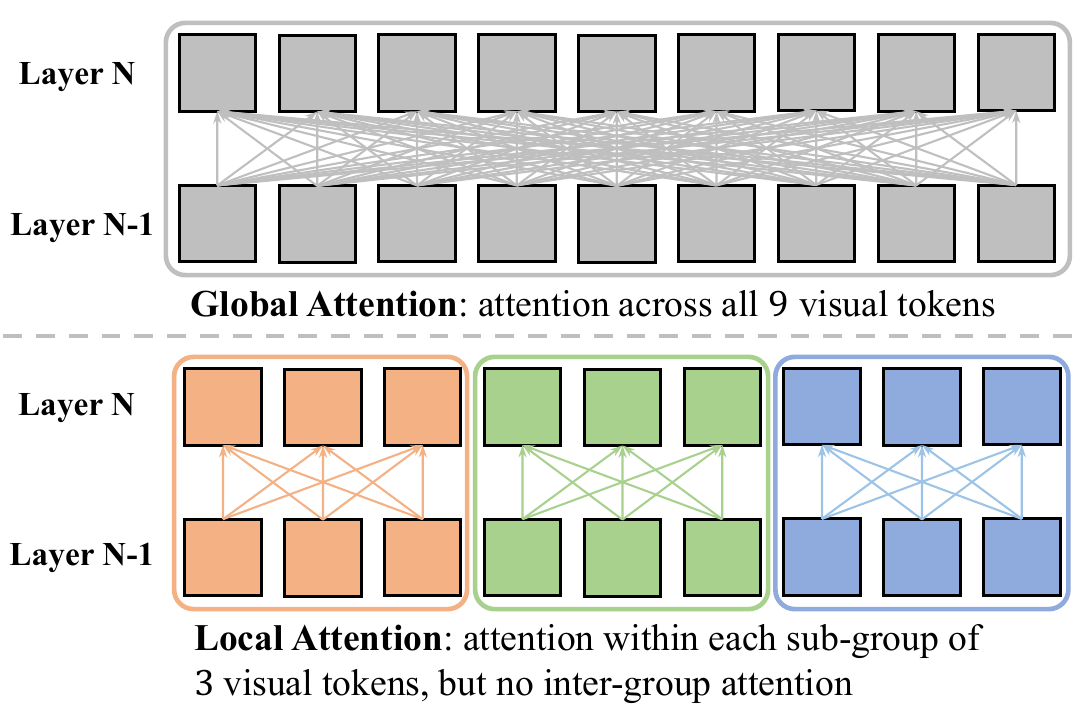}
    \caption{Local attention: each square is a visual token.}
    \label{fig-local-attention}
\end{figure}
\subsection{Building SRF Models}\label{sec-srf}

In this subsection, we discuss the details of the $\textsc{ToSRF}(\cdot)$ procedure -- how to build SRF models upon off-the-shelf LRF architectures. Since we take computation efficiency as a major performance metric in this paper, we aim to build SRF models with similar computation efficiency as vanilla LRF models. We will discuss approaches for both attention-based (e.g., ViT~\cite{vit}) and convolution-based (e.g., ResNet~\cite{resnet}) architectures.

\textbf{ViT-SRF -- a ViT~\cite{vit} variant.} To build an attention-based SRF model, we design a ViT variant named ViT-SRF. A vanilla ViT~\cite{vit} model leverages \textit{global attention} across all visual tokens to fine-tune visual features in each attention layer: the top of Figure~\ref{fig-local-attention} illustrates a global attention example where the attention operates across all $9$ visual tokens. The global attention makes each ViT feature have a large (global) receptive field. As a result, even if the attacker only corrupts one token, all tokens might be maliciously affected and corrupted. To enforce SRF and avoid complete feature corruption, we propose to use \textit{local attention} over a subset of visual tokens. At the bottom of Figure~\ref{fig-local-attention}, we provide an example where the attention operation is applied locally to sub-groups of visual tokens (each sub-group has $3$ tokens). With this local attention operation, one corrupted token can only affect tokens within the sub-group instead of all tokens.

\textit{Computation efficiency.} The ViT-SRF model can achieve similar computation efficiency as the vanilla ViT model. First, the number of attention operations can be reduced by our local attention operation. For example, the global attention example in Figure~\ref{fig-local-attention} requires ${{3\cdot3}\choose 2}=36$ attention pairs where the local attention only requires $3\cdot {3\choose2}=9$. On the other hand, the local attention requires additional operations like reorganizing tensor memory layout and adding/removing \texttt{[CLS]} tokens. We will show that our implementation of ViT-SRF has a similar overall inference speed as vanilla ViT in Section~\ref{sec-eval}.

\textbf{BagNet~\cite{bagnet} -- a ResNet~\cite{resnet} variant.} To build convolution-based SRF models, we leverage an off-the-shelf SRF architecture named BagNet~\cite{bagnet}. BagNet is based on the ResNet-50 architecture~\cite{resnet}; it achieves SRF by reducing the kernel sizes and strides of certain convolution and max-pooling layers to ones. BagNet was originally proposed for interpretable machine learning and later adopted as a building block for patch defenses (e.g., PatchGuard~\cite{xiang2021patchguard}, BagCert~\cite{bagcert}).

\textit{Computation efficiency.} Small kernel sizes and strides increase the size of the feature map as well as the overall computation costs. As a result, BagNet requires more computation than ResNet-50: we find that BagNet can be 1.5$\times$ slower than ResNet in our empirical evaluation (Section~\ref{sec-eval}). 
Nevertheless, we still categorize BagNet as a computationally efficient SRF model, since other convolution-based SRF models usually incur more than 10$\times$ computation overheads.

\textbf{Notes: other inefficient SRF models.} We note that there exist other \textit{inefficient} SRF architectures~\cite{levine2020randomized,li2022vip,salman2022certified,ecvit}. For example, a popular SRF strategy is to use an ensemble of vanilla classifiers: each classifier makes a prediction on a cropped small image region; the final prediction is generated via majority voting. This strategy was first proposed in De-randomized Smoothing~\cite{levine2020randomized} for ResNet, which incurs a 200+$\times$ slowdown. Later, this ensemble idea was adapted for ViT~\cite{vit} in Smoothed ViT~\cite{salman2022certified}, ECViT~\cite{ecvit}, and ViP~\cite{li2022vip} with different tricks to improve computation efficiency; however, these SRF models still require more than 10$\times$ more computation than vanilla ViT. In contrast, our ViT-SRF has similar computation efficiency as vanilla ViT and BagNet only incurs a 1.5$\times$ slowdown (Table~\ref{tab-undefended} in Section~\ref{sec-eval}).

\subsection{Robustness Certification}\label{sec-cert}

In this subsection, we discuss the robustness certification algorithm ($\textsc{PCURE-Certify}(\cdot)$ in Algorithm~\ref{alg-all}). We first discuss the correctness of two sub-procedures used in $\textsc{PCURE-Certify}(\cdot)$, i.e., $\textsc{Map}(\cdot)$ and $\textsc{SO-Cert}(\cdot)$, and then formally state and prove the correctness of our certification algorithm in Theorem~\ref{thm}.

\textbf{Sub-procedure $\textsc{Map}(\cdot)$.} Recall that $\textsc{Map}(\cdot)$ uses Equation~\ref{eqn-corruption} (from Section~\ref{sec-prelim-pgpc}) to calculate the maximum number of corrupted features and convert an image-space threat model $\cA_\cR$ to a feature-space threat model $\cA_\cR^f$. A correctly implemented $\textsc{Map}(\cdot)$ procedure provides the following proposition.

\begin{proposition}[Correctness of $\textsc{Map}(\cdot)$]\label{lemma1}
    Given a correctly implemented $\textsc{Map}(\cdot)$, an image-space threat model $\cA_\cR$, an SRF sub-model $\bM_\srf$, an input image $\bfx$, and the converted feature-space threat model $\cA_\cR^f=\textsc{Map}(\bM_\srf,\cA_\cR)$, we have: for any adversarial image $\bfx^\prime\in\cA_\cR(\bfx)$, its corresponding adversarial feature map $\bff^\prime = \bM_\srf(\bfx^\prime)$ is covered by the feature-space threat model $\cA_\cR^f$. Formally, we have $\forall \bfx^\prime\in \cA_\cR(\bfx),\bff^\prime = \bM_\srf(\bfx^\prime):\bff^\prime\in\cA_\cR^f(\bff), \bff=\bM_\srf(\bfx)$. 
\end{proposition}

\textbf{Sub-procedure $\textsc{SO-Cert}(\cdot)$.} A correctly implemented $\textsc{SO-Cert}(\cdot)$ ensures that the following proposition is true.

\begin{proposition}[Correctness of $\textsc{SO-Cert}(\cdot)$]\label{lemma2}
    Given a correctly implemented $\textsc{SO-Cert}(\cdot)$, an input tensor $\bff$, the model $\bM_\lrf$, the secure operation parameter $\cM$, the threat model $\cA_\cR^f$, and the ground-truth label $y$, if $\textsc{SO-Cert}(\bff,\bM_\lrf,\cM,\cA_\cR^f,y)$ returns \texttt{True}, we have certifiable robustness. Formally, we have $\textsc{SO-Cert}(\bff,\bM_\lrf,\cM,\cA_\cR^f,y)=\texttt{True} \implies \forall \bff^\prime\in\cA^f_\cR(\bff): \textsc{SO}(\bff^\prime,\bM_\lrf,\cM)= y$.
\end{proposition}

\textbf{Certification procedure $\textsc{PCURE-Certify}(\cdot)$.} With the two propositions discussed above, we can state and prove the correctness of the certification procedure below.
\begin{theorem}\label{thm}
    Given a clean image $\bfx$, its ground-truth label $y$, the \framework defense setting $\bM_\srf,\bM_\lrf,\cM$, the image-space patch threat model $\cA_\cR$, and correctly implemented $\textsc{Map}(\cdot)$ and $\textsc{SO-Cert}(\cdot)$ procedures (Propositions~\ref{lemma1} and \ref{lemma2}), if the certification procedure $\textsc{PCURE-Certify}(\cdot)$ returns \texttt{True}, we have certifiable robustness for this clean image. Formally, we have $\textsc{PCURE-Certify}(\bfx,y,\bM_{\srf},\bM_\lrf,\cM,\cA_\cR)=\texttt{True}\implies \forall \bfx^\prime\in\cA_\cR(\bfx): \textsc{PCURE-Infer}(\bfx^\prime,\bM_{\srf},\bM_\lrf,\cM)= y$
\end{theorem}

\begin{proof}

The correctness of two sub-procedures $\textsc{Map}(\cdot)$ and $\textsc{SO-Cert}(\cdot)$ gives us two useful propositions.
Proposition~\ref{lemma1} demonstrates that $\forall \bfx^\prime\in \cA_\cR(\bfx),\bff^\prime=\bM_\srf(\bfx^\prime)$, we have $\bff^\prime\in\cA_\cR^f(\bff)$; Proposition~\ref{lemma2} demonstrates that $\forall \bff^\prime\in\cA^f_\cR(\bff)$, we have $\textsc{SO}(\bff^\prime,\cdot)= y$, as long as $\textsc{SO-Cert}(\cdot)$ returns \texttt{True}. Combining two propositions together, we can derive that, $\forall \bfx^\prime\in \cA_\cR(\bfx), \bff^\prime=\bM_\srf(\bfx^\prime)$ we have $\textsc{SO}(\bff^\prime,\cdot)= y$, as long as $\textsc{SO-Cert}(\cdot)$ returns \texttt{True}. 

According to Lines~\ref{ln-certify-s}-\ref{ln-certify-e} of Algorithm~\ref{alg-all}, $\textsc{PCURE-Certify}(\cdot)$ returns \texttt{True} iff $\textsc{SO-Cert}(\cdot)$ returns \texttt{True}. Moreover, Lines~\ref{ln-infer-s}-\ref{ln-infer-e} show that $\textsc{PatchCURE-Infer}(\bfx^\prime,\cdot)= \textsc{SO}(\bff^\prime,\cdot)$. Therefore, we have proved that $\forall \bfx^\prime\in \cA_\cR(\bfx), \textsc{PCURE-Infer}(\bfx^\prime,\cdot)= y$, as long as $\textsc{PCURE-Certify}(\cdot)$ returns \texttt{True}.
\end{proof}

\textbf{Remark 1: certifiable robustness evaluation.} In our evaluation (Section~\ref{sec-eval}), we will apply the $\textsc{PCURE-Certify}(\cdot)$ procedure to labeled datasets and report \textit{certified robust accuracy} as the fraction of test images for which $\textsc{PCURE-Certify}(\cdot)$ returns \texttt{True}. This certified robust accuracy is our robustness evaluation metric.

\textbf{Remark 2: adaptive attacks vs. certifiable robustness.} Theorem~\ref{thm} ensures that the certified robust accuracy discussed above covers all possible attackers within a given threat model, i.e., $\forall \bfx^\prime\in \cA_\cR(\bfx)$. The threat model $\cA_\cR$ can capture an adaptive attacker who has full knowledge of the defense and uses a patch of a certain shape and size at all possible locations and with all possible patch content. Therefore, we can view certified robust accuracy as a lower bound of model accuracy against any (adaptive) attack within the threat model $\cA_\cR$. For example, if we consider an attacker who can use a 2\%-pixel square patch with any patch content at any image location, a certified robust accuracy of 61.6\% means that no adaptive attacker using the same patch shape and size can reduce the model accuracy below 61.6\% (guaranteed by Theorem~\ref{thm}). With this theoretical guarantee, our evaluation focuses on certified robustness instead of empirical robustness against \textit{concrete} adaptive attack algorithms.

\textbf{Remark 3: certification and inference procedures.} We note that our certification procedure $\textsc{PCURE-Certify}(\cdot)$ requires ground-truth labels to check the correctness of the model prediction for robustness evaluation. In contrast, our inference procedure $\textsc{PCURE-Infer}(\cdot)$ does not require a ground-truth label and thus can be deployed in the wild. The certifiable robustness evaluated on a labeled dataset (using the certification procedure) provides a robustness estimation for the model (inference procedure) deployed in the wild.

\subsection{\framework Instantiation}\label{sec-defense-discussion}

In this subsection, we discuss how to instantiate \framework with different parameter settings to approach the three-way trade-off problem. 

\textbf{\framework parameters.} There are four major parameters for the \framework defense. The \textit{first} is the splitting layer index $\splitk$ used in model splitting $\textsc{Split}(\cdot)$. As we choose a larger $\splitk$, we will split the model at a deeper layer, which could improve inference efficiency but could also make the LRF sub-model too shallow to achieve high utility for the end-to-end defense model. The \textit{second} parameter is the base model architecture $\bM_{\text{b}}$ used to construct SRF and LRF models. Different architectures usually have different properties. For example, ViT is shown more robust to occlusion~\cite{he2022masked}, which can be viewed as a non-adversarial pixel patch. The \textit{third} parameter is the receptive field size $\mathbf{rf}$, larger receptive fields can improve model utility but might hurt the certifiable robustness~\cite{xiang2021patchguard}. The fourth parameter is the parameter(s) of the secure operation layer; it controls the properties of the secure operation algorithm. 

Next, we focus on the model splitting parameter $\splitk$, which is the unique parameter introduced by our \framework framework, and discuss three different choices of $k$ for \framework instantiation. 
Recall that $k\in\{0,1,\cdots,L\}$, where $L$ is the number of all model backbone layers (excluding feature aggregation and classification layers), specifies the layer where we split the model and insert the secure operation.

\textbf{Case 1: Optimizing for computation efficiency. ($k=L$).} First, we aim to build efficient defense instances by setting $\splitk$ to its largest value $L$. In this case, we convert the entire model backbone into an SRF model and instantiate the LRF model as a combination of global average pooling plus a linear classification head. The computation overhead of the LRF model is much smaller than the SRF model (e.g., 500$\times$ difference). Even dozens of LRF calls generated from the secure operation layer will not have a significant impact on the model inference efficiency. 

\textit{Remark.} In this case, we can consider \framework reduced to a pure SRF defense like PatchGuard~\cite{xiang2021patchguard}. Nevertheless, we note that \framework's flexibility of combining different modules naturally leads to the new idea of applying an LRF-based secure operation, e.g., double-masking~\cite{xiang2022patchcleanser}, to an SRF-based feature map. In Figure~\ref{fig-main-comparison}, we have seen the benefits of this new combination: we build efficient \framework instances that outperform all but one existing (inefficient) defense in terms of robustness.

\textbf{Case 2: Optimizing for model utility and robustness ($k=0$).} In the second case, we set $k=0$ and reduce \framework to a pure LRF defense like PatchCleanser~\cite{xiang2022patchcleanser}. \framework with $k=0$ can match the utility and robustness of any existing defense, at the cost of relatively large computation overheads.

\textbf{Case 3: Interpolation between efficient defenses with high-robustness defenses ($0< k<L$).} Our final case study aims to leverage \framework to systematically adjust model performance in terms of robustness, utility, and efficiency. Recall that we can build state-of-the-art efficient defenses using $k=L$ and state-of-the-art defenses (without efficiency considerations) using $k=0$. Now we want to further build defenses whose robustness-utility performance can fill the gap between efficient \framework (top stars in Figure~\ref{fig-main-comparison}) and the state-of-the-art PatchCleanser~\cite{xiang2022patchcleanser} (the pentagon in Figure~\ref{fig-main-comparison}; a special \framework instance with $k=0$). We note that this task used to be hard: existing LRF defenses like PatchCleanser~\cite{xiang2022patchcleanser} do not have a design point that has similar inference efficiency as undefended models; existing efficient SRF defenses like PatchGuard~\cite{xiang2021patchguard} cannot sacrifice part of their efficiency in trade for better utility and robustness. However, with our \framework design, we can easily reach different performance points by varying $k$ between $0$ and $L$, as demonstrated in Figure~\ref{fig-main-comparison} in Section~\ref{sec-intro} and more analyses in Section~\ref{sec-eval}. This flexibility of tuning defense performance is useful in practice when we want to find the optimal defense instances given utility and efficiency constraints.

\section{Evaluation}\label{sec-eval}
In this section, we present our evaluation results. We first introduce our evaluation setup. Next, we will demonstrate that \framework (1) guides us to build efficient defenses that outperform all but one prior (mostly inefficient) defense, and (2) allows us to flexibly adjust the robustness, utility, and efficiency of defense models. Finally, we provide detailed analyses of our defense properties, including the effect of different defense parameters. Our source code is available at \url{https://github.com/inspire-group/PatchCURE}


\subsection{Setup}
\textbf{Dataset.} We focus on the ImageNet-1k dataset~\cite{imagenet}. It contains 1.3M training images and 50k validation images from 1000 classes. ImageNet-1k is a challenging dataset that witnesses an intense three-way trade-off between robustness, utility, and efficiency~\cite{survey}. It was also widely used for evaluation in prior works~\cite{xiang2021patchguard,levine2020randomized,zhang2020clipped,bagcert,salman2022certified,li2022vip,ecvit} and thus enables an easy and fair comparison between different defenses. We consider an image resolution of 224$\times$224 for a fair inference efficiency comparison. We also include additional evaluation results for CIFAR-10~\cite{cifar}, CIFAR-100~\cite{cifar}, and SVHN~\cite{svhn} in Section~\ref{sec-eval-detail}.

\textbf{Model architectures.} We consider ViT-B~\cite{vit} and ResNet-50~\cite{resnet} as our non-robust base models, and build ViT-SRF and BagNet as their SRF variants.

\textit{ViT-B and ViT-SRF.} ViT-B~\cite{vit} has 12 attention layers; we split between different attention layers, i.e., we select $k\in\{0,1,\cdots,12\}$. Moreover, we consider the ViT-B architecture with an input image size of $224\times224$: each visual token accounts for $16\times16$ image pixels, and there are $14\times14$ visual tokens in total. We consider local attention with a sub-group of $2\times2$, $14\times1$, and $14\times2$ visual tokens, which correspond to a receptive field size of $32\times32$, $224\times 16$, $224\times 32$ pixels, respectively. We name a ViT-SRF instance that uses local attention over $m\times n$ visual tokens as ViT\{m\}x\{n\}. We name a \framework instance that splits at $k^\text{th}$ layer and instantiates SRF model with ViT\{m\}x\{n\} as PCURE-ViT\{m\}x\{n\}-k\{k\}. For example, PCURE-ViT14x1-k9 stands for a \framework instance that splits the vanilla ViT-B at the $\text{9}^\text{th}$ attention layer and uses ViT-SRF with 14x1 local attention. 

\textit{ResNet-50 and BagNet.} ResNet-50~\cite{resnet} has 50 layers in total, and we consider $k\in\{0,1,\cdots,50\}$. Moreover, we consider BagNet with different receptive field sizes of $17\times17$, $33\times33$, and $45\times45$ pixels. We name a \framework instance that uses BagNet with a receptive field size of $m\times m$ and splits at the $k^\text{th}$ layer as {PCURE-BagNet\{m\}-k\{k\}}. For example, PCURE-BagNet33-k50 stands for a \framework instance that splits the vanilla ResNet-50 at the $\text{50}^\text{th}$ layer and uses BagNet33 as the SRF sub-model.

\textbf{Model training.} We note that \framework models without the secure operation layer are standard feedforward networks. Therefore, we can leverage off-the-shelf training techniques and recipes to train our \framework models. For ViT-based models (ViT or ViT-SRF), we take the pretrained weight from MAE~\cite{he2022masked} and follow the recipe of MAE fine-tuning to tune a model for the classification task on ImageNet-1k. For ResNet-based models (ResNet or BagNet), we follow Schedule B of the ResNet-strikes-back paper~\cite{strikeback} and train the model from scratch. We note that both ViT-based and ResNet-based models only use ImageNet-1k without additional data for fair comparison. During the training, we add random masks to the feature map of the $k^\text{th}$ model backbone layer; this mimics the masking operation used in our LRF secure operation layer (e.g., double-masking~\cite{xiang2022patchcleanser}). After we train our \framework models, we add the secure operation layer back for robust inference.

\textbf{Attack threat model.} Our main evaluation results focus on \textit{one} square patch that takes up to 2\% of the image pixels and can be placed anywhere on the image, i.e., a 32$\times$32 patch anywhere on the 224$\times$224 image. This is a popular evaluation setting used in most existing works~\cite{xiang2021patchguard,levine2020randomized,zhang2020clipped,bagcert,salman2022certified,li2022vip,ecvit}. We will additionally analyze model performance for large patches (up to 70\% of the image pixels) in Section~\ref{sec-eval-detail}. As discussed in Section~\ref{sec-cert} (Remark 2), our robustness certification procedure has accounted for \textit{any} adaptive attack strategy within the constraint $\cA_\cR$ defined in Section~\ref{sec-formulation}. Therefore, we only need to specify the threat model $\cA_\cR$ (but not concrete attack algorithms) for robustness evaluation.

\textbf{Evaluation metrics.} We focus on the three-way trade-off problem between certifiable robustness, model utility, and computation efficiency; therefore, we have three major evaluation metrics. For certifiable robustness, we report \textit{certified robust accuracy}, defined as the fraction of test images that the certification procedure ($\textsc{PCURE-Certify}(\cdot)$ of Algorithm~\ref{alg-all}) returns \texttt{True}. This is a lower bound of model accuracy against any adaptive attack within the given threat model (recall Remark 2 in Section~\ref{sec-cert}). For model utility, we report \textit{clean accuracy}, which is the standard model accuracy on clean test images without adversarial patches. For computation efficiency, we report empirical \textit{inference throughput} on clean images, which is defined as the number of images a model can make predictions within every second (including data loading and model feedforward). We use a batch size of 4 for the throughput evaluation -- we find large batch sizes do not significantly affect throughput in our experiment setting (Appendix~\ref{apx-flops}). We also report latency, estimated FLOPs, and memory footprint of different defense models in Appendix~\ref{apx-flops}. We conduct our experiments using PyTorch~\cite{pytorch} on one NVIDIA RTX A4000 GPU. We also discuss the implication of different hardware settings and application scenarios in Section~\ref{sec-discussion}.

\textbf{Comparison with prior defenses.} We compare our results with all prior certifiably robust defenses that can recover correct predictions (without abstention) and are scalable to the ImageNet-1k dataset: Clipped BagNet (CBN)~\cite{zhang2020clipped}, De-Randomized Smoothing (DRS)~\cite{levine2020randomized}, PatchGuard~\cite{xiang2021patchguard}, BagCert~\cite{bagcert}, PatchCleanser~\cite{xiang2022patchcleanser}, Smoothed ViT (S-ViT)~\cite{salman2022certified}, ECViT~\cite{ecvit}, and ViP~\cite{li2022vip}. We note that S-ViT, ECViT, and ViP propose similar algorithms by applying DRS to ViT, we sometimes group them into ``DRS+ViT'' to simplify the figure legend. In Appendix~\ref{apx-reproduce}, we provide additional details of how we obtain evaluation results for prior defenses.

\begin{table}[t]
    \centering
    \caption{Performance of undefended models}
    \label{tab-undefended}
    \begin{tabular}{l|r|r|r}
    \toprule
     \multirow{2}{*}{Models}  &   \multicolumn{2}{c|}{Accuracy (\%)}  &Throughput \\
         &  clean & robust & (img/s)\\
         \midrule
         ViT-B~\cite{vit,he2022masked} &83.7&0&191.7\\
         ResNet-50~\cite{resnet}&80.1&0&295.5\\
         BagNet-17~\cite{bagnet} &67.3&0&192.0\\
         BagNet-33~\cite{bagnet} &73.0&0&192.2\\
         BagNet-45~\cite{bagnet} &74.7&0&184.5\\
         {ViT14x2}&79.4&0&195.2\\
         {ViT14x1}&77.4&0&190.4\\
         {ViT2x2}&72.7&0&166.7\\
         \bottomrule
    \end{tabular}
\end{table}

\begin{table}[t]
    \centering
    \caption{Performance of different defenses against one 2\%-pixel patch anywhere on the ImageNet images (parentheses contain the relative performance compared to state-of-the-art PatchCleanser~\cite{xiang2022patchcleanser})}
    \label{tab-main}
    \resizebox{\linewidth}{!}
    {\begin{threeparttable}
    \begin{tabular}{l|r|r|r}
    \toprule
     \multirow{2}{*}{Defenses\tnote{\dag}}  &   \multicolumn{2}{c|}{Accuracy (\%)}  &{Throughput }   \\
         &  \multicolumn{1}{c|}{clean} & \multicolumn{1}{c|}{robust} & (img/s)\\

         \midrule
        {PCURE-ViT14x2-k12} & 78.3 (0.95)&44.2 (0.72)&\textbf{189.9 (90.4)}\\
        {PCURE-ViT14x1-k12} & 76.3 (0.92)&47.1 (0.77)&\textbf{182.0 (86.7)}\\
        {PCURE-ViT2x2-k12} & 71.3 (0.86)&46.8 (0.77)&\textbf{158.1 (75.3)}\\
    {PCURE-BagNet17-k50} & 65.4 (0.79)&44.1 (0.72)&\textbf{115.1 (54.8)}\\
        {PCURE-BagNet33-k50} & 70.8 (0.85)&42.2 (0.69)&\textbf{136.8 (65.1)}\\
        {PCURE-BagNet45-k50} & 72.4 (0.88)&34.9 (0.57)&\textbf{132.5 (63.1)}\\
        \midrule
        {PCURE-ViT14x2-k11} & 79.1 (0.96)&45.9 (0.75)&109.0 (51.9)\\
        {PCURE-ViT14x2-k10} & 79.8 (0.97)&46.2 (0.75)&77.2 (36.8)\\
        {PCURE-ViT14x2-k9} & 80.0 (0.97)&46.7 (0.76)&58.4 (27.8)\\
        {PCURE-ViT14x2-k6} & 81.8 (0.99)&47.5 (0.78)&34.4 (16.4)\\
        {PCURE-ViT14x2-k3} & \textbf{82.1 (0.99)}&47.8 (0.78)&23.9 (11.4)\\
        \midrule
        {PCURE-ViT14x2-k0} & \textbf{82.2 (1.00)}&47.8 (0.78)&19.7 (9.4) \\
        {PCURE-ViT14x1-k0} & \textbf{82.5 (1.00)}&53.8 (0.88)&8.3 (4.0)\\
        {PCURE-ViT2x2-k0} & \textbf{82.5 (1.00)}&\textbf{61.6 (1.00)}&2.0 (1.0)\\
         \midrule
         PatchCleanser~\cite{xiang2022patchcleanser}  (robust)\tnote{\S} &\textbf{82.5 (1.00)}&\textbf{61.1 (1.00)}&2.1 (1.0)\\ 
         PatchCleanser~\cite{xiang2022patchcleanser} (efficient)&82.0 (0.99)&55.1 (0.90)&12.5 (5.9)\\
         ECViT~\cite{ecvit}&78.6 (0.95)&41.7 (68.2)&2.25 (1.1)\\
         ViP~\cite{li2022vip} (robust)&75.8 (0.92)&40.4 (66.1)&0.8 (0.4)\\
         ViP~\cite{li2022vip} (efficient)&75.3 (0.91)&38.3 (0.63)&7.7 (3.7)\\
         S-ViT~\cite{salman2022certified} (robust)&73.2 (0.89)&38.2 (0.63)&0.8 (0.4)\\
         S-ViT~\cite{salman2022certified} (efficient)&67.3 (0.82)&33.0 (0.54)&20.5 (9.7)\\
         PatchGuard~\cite{xiang2021patchguard}&54.6 (0.66)&26.0 (0.43)&\textbf{177.4} (\textbf{84.5})\\
         BagCert~\cite{bagcert}&45.3 (0.55)&22.7 (0.37)&\textbf{181.4} (\textbf{86.4})\\
         DRS~\cite{levine2020randomized}&44.4 (0.54)&14.0 (0.23)&1.5 (0.7)\\
         CBN~\cite{zhang2020clipped}&49.5 (0.60)&7.1 (0.11)&\textbf{181.5} (\textbf{86.4})\\
         \bottomrule
    \end{tabular}
          \begin{tablenotes}
      \item[$\dag$] For each prior defense without multiple instances, we report its most robust and most efficient instances, denoted with ``(robust)'' and ``(efficient)".
      \item[$\S$] The PatchCleanser numbers in its original paper are slightly higher because they use additional ImageNet-21k~\cite{imagenet} data for pretraining.
      
  \end{tablenotes}
    \end{threeparttable}}
\end{table}

\subsection{Main results}\label{sec-eval-main}

In Table~\ref{tab-undefended} and Table~\ref{tab-main}, we report the main evaluation results for undefended models and different defenses on the ImageNet-1k dataset against one 2\%-pixel square patch anywhere on the image. We note that many defenses have multiple defense instances; we report the performance of their most \textit{robust} and most \textit{efficient} instances if efficiency greatly affects their robustness performance. In Table~\ref{tab-main}, we additionally report the relative performance compared with the most robust prior defense, i.e., PatchCleanser~\cite{xiang2022patchcleanser}.

\textbf{\framework builds defense instances that have similar efficiency as undefended models.} First, we analyze efficient defense instances. As discussed in Section~\ref{sec-defense-discussion}, we will set $k=L$. We report performance for {PCURE-ViT-14x2-k12}, {PCURE-ViT-14x1-k12}, {PCURE-ViT-2x2-k12}, {PCURE-BagNet33-k50}, {PCURE-BagNet17-k50}, and {PCURE-BagNet45-k50} in Table~\ref{tab-main}. For ViT-based defenses, we can see that the inference throughput is similar to that of vanilla ViT-B (Table~\ref{tab-undefended}). For ResNet/BagNet-based defenses, we find that our defenses are 2$\times$ slower than ResNet-50. Nevertheless, our BagNet-based defenses are still faster than most prior works, and we still categorize them as efficient defenses. In addition to the high inference efficiency, \framework also achieves decent model utility and robustness. For example, PCURE-ViT14x2-12 has a 78.3\% clean accuracy for the challenging 1000-class ImageNet classification task and 44.2\% certified robust accuracy against one 2\%-pixel square patch anywhere on the image. 

\begin{figure*}
\centering
\begin{minipage}[b]{0.30\linewidth}
\includegraphics[width=\linewidth]{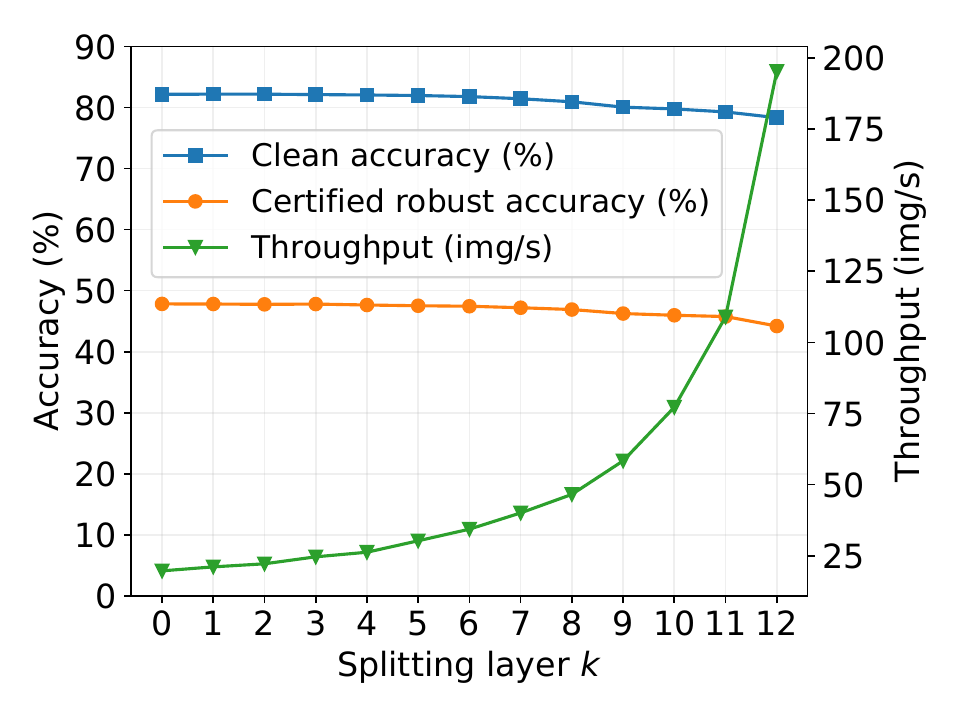}
    \caption{Effect of the splitting layer $k$ on {ViT14x2}-based \framework}
    \label{fig-splitk}
\end{minipage}%
\quad
\begin{minipage}[b]{0.34\linewidth}
\includegraphics[width=\linewidth]{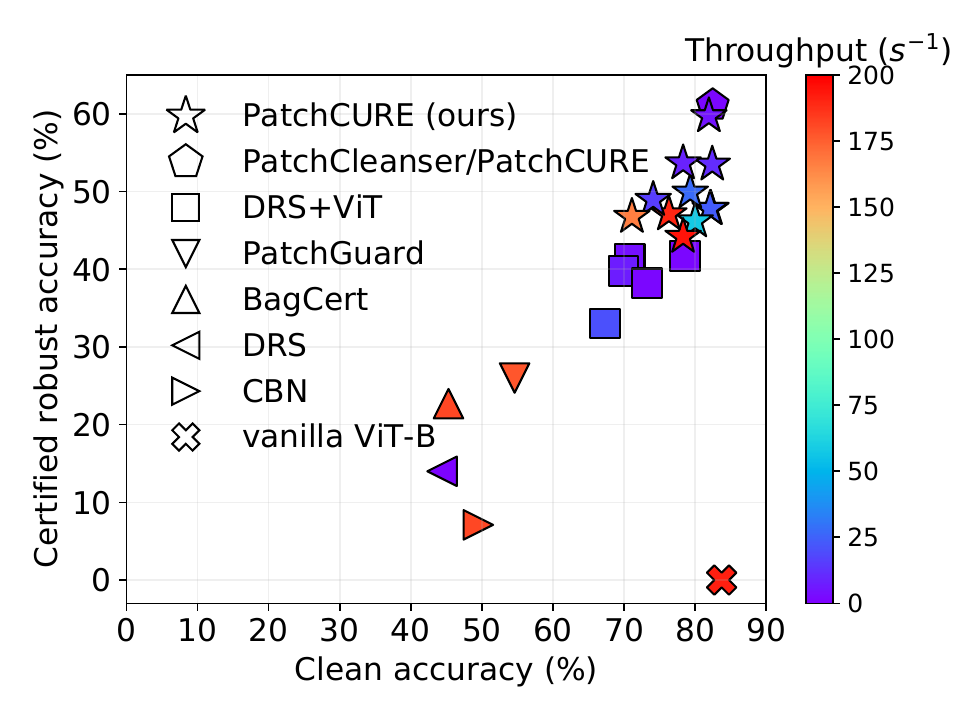}
    \caption{Comparison between \framework and prior defenses}
    \label{fig-comparison-eval}
\end{minipage}%
\quad
\begin{minipage}[b]{0.32\linewidth}
\includegraphics[width=\linewidth]{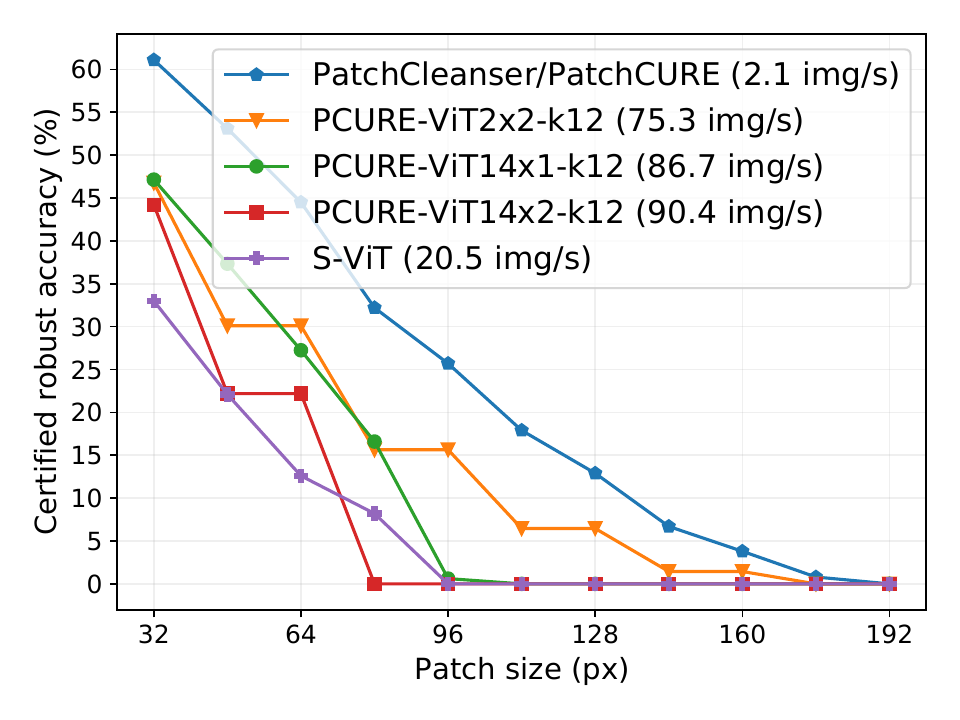}
    \caption{\framework robustness against larger patches}
    \label{fig-large-patch}
\end{minipage}%
\end{figure*}

\textbf{\framework provides a systematic way to balance the three-way trade-off between robustness, utility, and efficiency.} Second, we aim to demonstrate \framework's flexibility to balance defense performance by adjusting the parameters $k$. In Table~\ref{tab-main} and Figure~\ref{fig-splitk}, we report the performance of PCURE-ViT14x2 instantiated with different $k$. As we split the model at a deeper layer (larger $k$), the inference throughput greatly improves while the clean accuracy and the certified robust accuracy gradually decrease. This is because a deeper splitting point leads to a shallower LRF sub-model and thus reduces the cost of the secure operation layer (which calls the LRF sub-model multiple times). Meanwhile, a deeper splitting point slightly decreases the model capacity and leads to drops in clean accuracy and certified robustness accuracy.

\textbf{\framework builds efficient defenses that outperform all but one (mostly inefficient) prior defenses and can achieve the best robustness/utility by sacrificing part of the computation efficiency.} Now, we compare \framework with prior defenses. In addition to Table~\ref{tab-main}, we visualize the performance of different \framework instances and prior defenses in Figure~\ref{fig-comparison-eval}.\footnote{There are a large number of different \framework instances. Table~\ref{tab-main} and Figure~\ref{fig-comparison-eval} only report two partially overlapped subsets of \framework instances for simplicity and better visual appearance.} First, we can see that efficient \framework instances (e.g., the first three rows in Table~\ref{tab-main} and the red stars in Figure~\ref{fig-comparison-eval}) achieve state-of-the-art utility and robustness performance among defenses with small defense overheads. For example, compared to the best efficient defense PatchGuard~\cite{xiang2021patchguard}, our ViT-based defenses have 16.7-23.7\% absolute improvements in clean accuracy and 18.2-21.1\% absolute improvement in certified robust accuracy; our BagNet-based defenses have 10.8-17.8\% absolute clean accuracy improvements and 8.9-18.1\% absolute certified robust accuracy improvements. Moreover, we note that the robustness and utility performance of our efficient defenses is also comparable to, or even surpasses, many inefficient defenses. For example, our ViT-based defenses (red stars in Figure~\ref{fig-comparison-eval}) outperform defenses from the DRS-ViT family~\cite{salman2022certified,li2022vip,ecvit} (purple squares in Figure~\ref{fig-comparison-eval}). 

The only existing defense that significantly outperforms our efficient instances is PatchCleanser~\cite{xiang2022patchcleanser}. However, we note that PatchCleanser is simply a design point of \framework with $k=0$; \framework can easily reach the performance of PatchCleanser by sacrificing part of its computation efficiency. Figure~\ref{fig-comparison-eval} demonstrates that we can build a family of \framework instances (with different $k$ and SRF sub-models) that fill in the gap between PatchCleanser (purple pentagon) and efficient \framework instances (red stars). 
Moreover, we note that the \framework instances also achieve the best robustness and utility performance across all different efficiency levels (denoted by different colors in Figure~\ref{fig-comparison-eval}).
This \framework defense family allows us to easily find the optimal defenses that satisfy certain utility or computation constraints in practice.

\subsection{Detailed Analyses}\label{sec-eval-detail}

In this section, we provide additional analyses of our \framework defenses. We will discuss the properties of SRF models used in \framework, \framework performance when facing larger patches, \framework performance on additional datasets, and the effects of different \framework parameters.

\textbf{Our undefended SRF models have similar inference efficiency as off-the-shelf models with a small utility drop.} In Section~\ref{sec-srf}, we discussed how to build SRF models using off-the-shelf CNN and ViT. In Table~\ref{tab-undefended}, we report the performance of vanilla undefended SRF models. First, we can see that ViT-based models have similarly high computation efficiency as vanilla ViT models. For {PCURE-ViT14x2}, the inference throughput is even slightly higher than vanilla ViT. Moreover, we note that as we use a smaller receptive field (from 14x2 to 14x1 to 2x2 visual tokens), the clean accuracy of ViT-SRF decreases as expected; meanwhile, the inference throughput is also slightly affected because a smaller receptive field leads to more sub-groups of visual tokens, and each sub-group incurs additional computation costs (conversion between to the sub-group-style token layout and the original-style visual token layout). Second, for BagNet -- the ResNet-based SRF model -- we find that both clean accuracy and inference throughput are moderately affected. This is because smaller convolution kernels and strides result in a larger feature map and thus more computation. This is also why we find that the ViT-based \framework achieves better performance than the BagNet-based \framework in Section~\ref{sec-eval-main}. For the rest of this subsection, we will focus on analyzing ViT-based defenses.

\textbf{Different \framework instances exhibit different robustness against larger patches.} In Section~\ref{sec-eval-main}, we report certified robust accuracy for one 2\%-pixel square patch anywhere on the ImageNet image, because this is a popular evaluation setup for benchmark comparison~\cite{xiang2021patchguard,levine2020randomized,zhang2020clipped,bagcert,salman2022certified,li2022vip,ecvit}. In Figure~\ref{fig-large-patch}, we report the certified robust accuracy of different \framework instances for different patch sizes (up to 192$\times$192 on the 224$\times$224 image). As shown in Figure~\ref{fig-large-patch}, the certified robust accuracy decreases as the patch size increases. Notably, different defense instances have different sensitivity to larger patches. The most efficient defense, {PCURE-ViT14x2-k12}, is most sensitive to larger patches due to the relatively large receptive field size of its SRF model ({ViT14x2} has a receptive field of 224$\times$32). In contrast, {ViT2x2} has a smaller receptive field of 32$\times$32, and its \framework defense is more robust to larger patches. Moreover, we plot the results for PatchCleanser~\cite{xiang2022patchcleanser} and S-ViT~\cite{salman2022certified}, the two best-performing open-source defenses. First, we note that PatchCleanser can be viewed as a special \framework instance that uses an identical mapping layer (with a receptive field size of 1$\times$1) as the SRF sub-model. It has the best robustness to larger patches but smaller throughput. Second, we can also see that our efficient \framework instances beat the robustness of S-ViT in almost all cases, at a much larger inference throughput.

\begin{table}[t]
    \centering
    \caption{Defense performance of PCURE-ViT14x2-k6 for different datasets (one 2\%-pixel patch on the resized 224$\times$224 images)}
    \label{tab-cifar}
    \begin{tabular}{l|r|r|r}
    \toprule
        \multirow{2}{*}{Dataset} & \multicolumn{2}{c|}{Accuracy (\%)}  &Throughput \\
          & clean & robust & (img/s)\\
         \midrule
         CIFAR-10~\cite{cifar}&96.8&70.3&44.4\\
         CIFAR-100~\cite{cifar}&85.2&40.2&33.5\\
         SVHN~\cite{svhn}&95.7&33.8&39.8\\
         \bottomrule
    \end{tabular}
\end{table}

\textbf{\framework performs well on different datasets including CIFAR-10~\cite{cifar}, CIFAR-100~\cite{cifar}, and SVHN~\cite{svhn}.} In Section~\ref{sec-eval-main}, we demonstrate that \framework works well for the challenging ImageNet dataset. In this analysis, we aim to demonstrate that our \framework works for other benchmark datasets as well. In Table~\ref{tab-cifar}, we report the defense performance of PCURE-ViT14x2-k6 for three datasets CIFAR-10~\cite{cifar}, CIFAR-100~\cite{cifar}, and SVHN~\cite{svhn}. We resize the image to 224$\times$224 to fit the input size of the ViT model and report certified robust accuracy for one 2\%-pixel square patch anywhere on the image. As shown in the table, different \framework instances work well for different datasets. This further demonstrates the general applicability of our \framework defenses. In Appendix~\ref{apx-cifar}, we include additional evaluation results for different \framework instances and prior defenses on the CIFAR-10 dataset.

\textbf{\framework parameter selection.} As discussed in Section~\ref{sec-defense-discussion}, \framework has four parameters: $k$ (splitting layer), model architecture, SRF size, and the secure operation parameter. We need to understand their effect to properly select parameters in practice. The effect of the first parameter $k$ has been already discussed in Figure~\ref{fig-splitk} and Section~\ref{sec-eval-main}: it is the most important knob in \framework to balance the three-way trade-off. In practice, we can set $k$ to the largest value that meets the computation efficiency requirement. The effects of the model architecture and receptive field size have been implicitly discussed in our earlier experiments. For model architecture, Table~\ref{tab-undefended} demonstrates that ViT has an advantage over ResNet in their SRF model's utility and efficiency; this advantage is also reflected in the defense performance of ViT and ResNet-based defenses (Table~\ref{tab-main}). We attribute this difference to the recent findings that ViT is more robust to input masking~\cite{he2022masked}.
The effect of different receptive field sizes is also demonstrated in Table~\ref{tab-main} and Figure~\ref{fig-large-patch}. Table~\ref{tab-main} shows that ViT-SRF with a larger receptive field has higher clean accuracy and slightly better inference efficiency. However, larger receptive fields make the defenses more vulnerable to larger patches (Figure~\ref{fig-large-patch}). Therefore, we need to carefully choose the defense parameters based on the application scenarios. In practice, we can choose the smallest receptive field size that meets the model utility requirement to get optimal robustness. The effect of the secure operation parameter depends on the chosen secure operation algorithm; we provide a quantitative analysis in Appendix~\ref{apx-pc-exp}.

\begin{table}[t]
    \resizebox{\linewidth}{!}
      {
\begin{threeparttable}
    \centering
    \caption{\framework for \textit{attack detection} on ImageNet against one 2\%-pixel square patch}
    \label{tab-det}

       \begin{tabular}{l|r|r|r}
    \toprule
     \multirow{2}{*}{Models}  &   \multicolumn{2}{c|}{Accuracy (\%)}  &Throughput \\
         &  clean & robust & (img/s)\\
         \midrule
         {PCURE-ViT14x2-k12}&64.7&64.7&196.8\\
         {PCURE-ViT14x1-k12}&63.5&63.5&192.3\\
         {PCURE-ViT2x2-k12}&59.5&59.5&166.5\\
         \midrule
         ViT + MR~\cite{xiang2022patchcleanser,mccoyd2020minority}&74.3&74.3&2.5\\
         ViP~\cite{li2022vip}&74.6&74.6&2.5\\
         ScaleCert~\cite{han2021scalecert}&58.5&55.4&NA\tnote{\dag}\\
         PatchGuard++~\cite{xiang2021patchguard2}&49.8&49.8&151.1\\
         \bottomrule
         
    \end{tabular}
          \begin{tablenotes}
      \item[$\dag$] We did not evaluate  ScaleCert~\cite{han2021scalecert} due to the lack of open-source implementation.
  \end{tablenotes}
    \end{threeparttable}}
\end{table}

\section{Discussion}\label{sec-discussion}

In this section, we discuss \framework's compatibility with different secure operation algorithms, different model sizes, application scenarios, limitations,  and future work directions.

\textbf{Compatibility with different secure operation algorithms.} We propose \framework as a general defense framework with three modules: the SRF model, the LRF model, and the secure operation. In Section~\ref{sec-eval}, we extensively analyze different design choices of SRF and LRF models. Nevertheless, we keep secure operation the same 
as double-masking~\cite{xiang2022patchcleanser} because it is the state-of-the-art secure operation algorithm for prediction recovery. In this discussion, we aim to showcase \framework's compatibility with a different secure operation algorithm, namely the Minority Reports (MR) defense~\cite{mccoyd2020minority}. MR is another LRF technique that aims to detect an ongoing patch attack. It applies masks to all possible image locations and checks the consistency in model predictions on different masked images. If there is a unanimous consensus among all masked predictions, MR returns the agreed prediction label; otherwise, it issues an attack alert.\footnote{We note that this attack-detection-based defense has a weaker notion than prediction-recovery-based defense like \framework and PatchCleanser~\cite{xiang2022patchcleanser}.} In Table~\ref{tab-det}, we report the performance of \framework with MR-style secure operation and compare it with other attack-detection-based defenses. The table demonstrates that our efficient \framework instances achieve decent robustness and utility performance with high inference speed.

\textbf{\framework with different model sizes.} In Section~\ref{sec-eval}, we demonstrate that \framework can flexibly tune the model performance without significantly altering the model architecture and size. We note that \framework technique is also compatible with (and orthogonal to) strategies like using models with different sizes. In Table~\ref{tab-model-size}, we report an example using \framework with ViT-Base and ViT-Large. We can see that PCURE-ViT-B-14x2-k9 can match the robustness, utility, and throughput of PCURE-ViT-L-14x2-k24 (ViT-L has 24 attention layers in total), but with smaller memory consumption. This also demonstrates the benefits of introducing \framework's knobs. 
In practice, developers can combine \framework with orthogonal techniques like model compression and pruning.

\begin{table}[t]
    \centering
    \caption{\framework with ViT-Base and ViT-Large}
    \label{tab-model-size}
    \resizebox{\linewidth}{!}{
    \begin{tabular}{l|r|r|r|r}
    \toprule
Defense & Clean& Certified& Throughput& Memory (MB)\\
\midrule
 ViT-B-14x2-k12   & 78.3 & 44.2      & 189.9        & 362.0            \\
 ViT-B-14x2-k9     & 80.0 & 46.2      & 58.4          & 515.6            \\
 ViT-L-14x2-k24  & 80.3 & 47.7      &  58.0         & 1203.2           \\
\bottomrule
    \end{tabular}}

\end{table}


\textbf{\framework application scenarios.} In this paper, we focus on the general problem of adversarial patch attacks/defenses, which can be applied to many different applications like autonomous driving~\cite{rp2}, face authentication~\cite{wei2022adversarial}, surveillance~\cite{xu2020adversarial}, and cashierless self-checkout~\cite{hofman2023x}, which have different hardware settings and also different levels of sensitivity to latency. Our evaluation setting in Section~\ref{sec-eval} is more suited for cloud deployments since we use A4000, a common workstation GPU, for experiments. We demonstrate that \framework can have a similar inference speed as vanilla models under the same hardware setting (191.7 img/s for ViT-B vs. 189.9 img/s for PCURE-ViT14x2-k12). This implies that, in time-sensitive applications, \framework can meet the throughput/latency requirement as the vanilla model does. An interesting direction of future works is to concretely implement \framework with other model compression techniques for embedded edge devices.

\textbf{Limitations and future directions.} The certifiable robustness of \framework relies on its secure operation layer. As a result, \framework instances will inherit limitations from the secure operation algorithm. For example, if we use the Minority Reports algorithm~\cite{mccoyd2020minority}, we can only achieve robustness for attack detection. If we use the double-masking algorithm~\cite{xiang2022patchcleanser}, our defense requires additional patch information such as an estimation of patch size and shape. How to relax these assumptions is an important future work direction. Nevertheless, we note that \framework provides a systematic way to build defense instances with different properties. The compatibility of the framework implies its potential for future improvements: given any progress in the design of SRF models, LRF models, or secure operations, we can leverage these advancements to build stronger \framework defense instances. Moreover, another interesting direction is to incorporate \framework, a test-time defense, with training-time certified defenses~\cite{gowal2018effectiveness,mirman2018differentiable} to further improve robustness.

\section{Related Work}\label{sec-related-work}
\textbf{Adversarial Patch Attacks.} Brown et
al.~\cite{brown2017adversarial} introduced the first adversarial patch attack; they demonstrated that an attacker can use a physically printed adversarial patch to induce targeted misclassification. The physical realizability of patch attacks has drawn great attention from the machine learning security community. In a concurrent work, Karmon et al.~\cite{karmon2018lavan} explored a similar concept called the Localized and Visible Adversarial Noise (LaVAN) attack, which focuses on the digital-domain patch attack. Many more patch attacks have been proposed for different attack scenarios~\cite{yang2020patchattack,lovisotto2022give,liu2018dpatch,sehwag2018not,shan2021real,chen2022shape,xu2020adversarial,wu2019making}.
In this paper, we focus on attacks against image classification models.

\textbf{Adversarial Patch Defenses.} To defend against adversarial patch attacks, both empirically robust and certifiably robust defenses have been proposed. Empirical defenses~\cite{hayes2018visible,naseer2019local,wu2019defending,rao2020adversarial,Mu2021defending,saha2020role,chou2020sentinet,tarchoun2023jedi,liu2022segment,xu2023patchzero} are usually based on heuristics and lack formal robustness guarantees; in contrast, certifiably robust defenses~\cite{chiang2020certified,zhang2020clipped,levine2020randomized,xiang2021patchguard,xiang2022patchcleanser,bagcert,mccoyd2020minority,salman2022certified,li2022vip,ecvit} can claim robustness in a provable manner. Chiang et al.~\cite{chiang2020certified} proposed the first certifiably robust defense for patch attacks via Interval Bound Propagation (IBP)~\cite{gowal2018effectiveness,mirman2018differentiable}. This defense is too computationally intense to scale to the ImageNet dataset, so we did not include it in Section~\ref{sec-eval} for comparison. Zhang et al.~\cite{zhang2020clipped} proposed Clipped BagNet (CBN) that clips BagNet~\cite{bagnet} features for certifiable robustness. Levine et al.~\cite{levine2020randomized} proposed De-Randomized Smoothing (DRS) that performs majority voting on model predictions on different small cropped images. We note that DRS was further improved by Smoothed ViT~\cite{salman2022certified}, ECViT~\cite{ecvit}, and ViP~\cite{li2022vip} using the Vision Transformer (ViT)~\cite{vit} architecture. Xiang et al.~\cite{xiang2021patchguard} proposed PatchGuard shortly after DRS as a general defense framework that uses a model with small receptive fields (SRF) for feature extraction and performs secure feature aggregation for robust predictions. The idea of SRF is widely used in many defenses~\cite{levine2020randomized,salman2022certified,li2022vip,ecvit,bagcert,han2021scalecert}. Our \framework also leverages the SRF idea to improve computation efficiency. After PatchGuard, Metzen et al.~\cite{bagcert} proposed BagCert, in which they modified vanilla BagNet~\cite{bagnet} for robust predictions.  Xiang et al.~\cite{xiang2022patchcleanser} later proposed PatchCleanser with a double-masking algorithm that reliably removes the patch from the input image. In contrast to other works, PatchCleanser does not rely on SRF models and achieves state-of-the-art certifiable robustness while maintaining high model utility (e.g., 1\% drops from undefended models). In a concurrent work of this paper, PatchCleanser is further improved with a new model training technique~\cite{saha2023revisiting}. However, we note that PatchCleanser requires expensive model predictions on multiple masked images and incurs a large computation overhead. In this paper, we propose \framework to approach the three-way trade-off between robustness, utility, and efficiency. Moreover, we note that our general \framework design subsumes PatchCleanser by setting $k=0$.

In addition to the defenses discussed above, there are other certifiably robust defenses for attack detection~\cite{mccoyd2020minority,han2021scalecert,xiang2021patchguard2}. These defenses alert and abstain from making predictions when they detect an ongoing attack, which achieves a weaker robustness property compared to defenses that can recover correct predictions without any abstention. In Section~\ref{sec-discussion}, we demonstrated \framework's compatibility with one of the LRF-based attack-detection algorithms, Minority Reports~\cite{mccoyd2020minority}. Finally, there is another line of certifiably robust defenses that study harder vision tasks such as object detection~\cite{xiang2021detectorguard,xiang2023objectseeker} and semantic segmentation~\cite{yatsura2022certified}.
We refer the readers to a survey paper~\cite{survey} for more discussions on certified patch defenses.

\textbf{Other Adversarial Example Attacks and Defenses.} There are many existing paper studies adversarial example attacks defenses for different threat models~\cite{szegedy2013intriguing,biggio2013evasion,goodfellow2014explaining,papernot2016limitations,meng2017magnet,xu2017feature,carlini2017towards,madry2017towards,wong2017provable,cohen2019certified,xiang2019generating}. We focus on the adversarial patch attacks because they can be realized in the physical world and impose a threat to the cyber-physical world. 

\section{Conclusion}\label{sec-conclusion}
In this paper, we proposed \framework as a general defense framework to build certifiably robust defenses against adversarial patch attacks. \framework is the first to explicitly approach the three-way trade-off problem between certifiable robustness, model utility, and computation efficiency. We demonstrated that \framework enables us to build state-of-the-art efficient defense instances, and provides sufficient knobs to adjust the defense performance across the three dimensions. We note that \framework is a general defense framework compatible with any SRF and LRF model architectures as well as secure operation algorithms -- advancements in SRF/LRF model architectures and secure operations will also lead to stronger \framework defense instances.

\section*{Acknowledgements}%
We are grateful to the anonymous shepherd and reviewers at USENIX Security 2024 for their valuable feedback.
This work was supported in part by the National Science Foundation under grant CNS-2131938 and the Princeton SEAS Innovation Grant.

\bibliographystyle{plain}
\bibliography{reference.bib}

\appendix

\section{Notes on Reproducing Prior Defenses}\label{apx-reproduce}

In our evaluation, we report three performance metrics for each defense: clean accuracy, certified robust accuracy, and inference throughput. In this section, we provide additional details on how we obtain these performance metrics for prior defenses. In Table~\ref{tab-reproduce}, we categorize prior defenses into three groups based on their reproducibility. The first group (top rows of the table) contains defenses with open-source code and model weights. For this group, we directly adopt the publicly available source code for evaluation. The second group corresponds to defenses that do not have publicly available source code (nor the model weights), but their defense inference logic is easy for us to re-implement. We find that different model weights (random weights or the weights we trained) do not greatly affect the empirical inference throughput; therefore, we report the inference throughput evaluated based on our re-implementation and copy clean accuracy and certified robust accuracy from their original papers. The third group only contains the ScaleCert~\cite{han2021scalecert} defense. Since ScaleCert is for attack detection, which is not the focus of our main evaluation, we did not re-implement its defense logic nor report its inference throughput performance.

\begin{table}
    \centering
    \caption{Reproducibility of prior defenses}
    \label{tab-reproduce}
    \begin{tabular}{l|c|c}
    \toprule
        \multirow{2}{*}{Defenses} & \multicolumn{1}{l|}{Open-source} & \multicolumn{1}{l}{Re-implement} \\
         & code\&weights?&inference logic?\\
        \midrule
    PatchCleanser~\cite{xiang2022patchcleanser}     &\cmark  & \cmark \\
    PatchGuard~\cite{xiang2021patchguard}  &  \cmark  & \cmark \\
    PatchGuard++~\cite{xiang2021patchguard2}  & \cmark  & \cmark \\
    S-ViT~\cite{salman2022certified}   &  \cmark  & \cmark \\
    DRS~\cite{levine2020randomized}  &  \cmark  & \cmark \\
    CBN~\cite{zhang2020clipped}  &  \cmark  & \cmark \\
    MR~\cite{mccoyd2020minority}   & \cmark  & \cmark \\
    \midrule
    ECViT~\cite{ecvit}     & \xmark & \cmark \\
    ViP~\cite{li2022vip}  & \xmark & \cmark \\
    BagCert~\cite{bagcert}   & \xmark & \cmark \\
    \midrule
    ScaleCert~\cite{han2021scalecert}    & \xmark & \xmark \\
    \bottomrule
    \end{tabular}

\end{table}

\begin{table}[t]
    \centering
    \caption{Throughput results (img/s) with different batch sizes}
    \label{tab-efficiency-batch-size}
      \resizebox{\linewidth}{!}
    {
    \begin{tabular}{l|r|r|r|r}
    \toprule
     Batch size    &  1&4&8&32\\
         \midrule
 ViT-B    & 113.4& 190.6& 189.3&  199.9\\
 ViT14x2  & 111.5& 192.9& 193.8& 203.9\\
 PCURE-ViT14x2-k12& 110.7& 186.3& 188.9&  200.9\\
 PatchCleanser          & 2.0    & 2.1    & 2.1    & 2.1     \\
 S-ViT                         & 0.6    & 0.8    & 0.8       & 0.8 \\

         \bottomrule
    \end{tabular}}
\end{table}

\begin{table}[t]
    \centering
    \caption{Performance for different models and defenses}
    \label{tab-flops}
        \resizebox{\linewidth}{!}
    {
    \begin{tabular}{l|r|r|r|r|r|r}
    \toprule
     \multirow{2}{*}{Models}  &   \multicolumn{2}{c|}{Accuracy (\%)}  &\multicolumn{1}{c|}{Throughput} & \multicolumn{1}{c|}{Latency}&\multicolumn{1}{c|}{FLOP}&\multicolumn{1}{c}{Memory}\\
         &  clean & robust & (img/s)&(ms)&($\times10^9$)& (MB)\\
         \midrule
         ViT-B~\cite{vit} &83.7&0&191.7&7.5&17.6&360.6\\
         ResNet-50~\cite{resnet}&80.1&0&295.5&11.8&4.1&254.8\\
         BagNet-17~\cite{bagnet} &67.3&0&192.0&11.3&15.8&311.8\\
         BagNet-33~\cite{bagnet} &73.0&0&192.2&11.0&16.4&320.2\\
         BagNet-45~\cite{bagnet} &74.7&0&184.5&12.9&16.1&318.2\\
         {ViT14x2}&79.4&0&195.2&9.2&17.6&362.0\\
         {ViT14x1}&77.4&0&190.4&8.2&18.1&362.4\\
         {ViT2x2}&72.7&0&166.7&8.8&21.1&368.1\\
         \midrule
        {PCURE-ViT14x2-k12} & 78.3&44.2&{189.9}&10.3&17.6&362.0\\
        {PCURE-ViT14x1-k12} & 76.3&47.1&{182.0}&9.8&18.1&362.4\\
        {PCURE-ViT2x2-k12} & 71.3&46.8&{158.1}&9.7&21.1&368.2\\
        {PCURE-BagNet17-k4} & 65.4&44.1&{115.1}&14.5&17.1&723.2\\
        {PCURE-BagNet33-k4} & 70.8&42.2&{136.8}&13.1&17.5&517.4\\
        {PCURE-BagNet45-k4} & 72.4&34.9&{132.5}&14.8&17.2&465.9\\
        \midrule
        {PCURE-ViT14x2-k11} & 79.1&45.9&109.0&13.8&27.4&513.9\\
        {PCURE-ViT14x2-k10} & 79.8&46.2&77.2&16.5&38.6&511.7\\
        {PCURE-ViT14x2-k9} & 80.0&46.7&58.4&21.1&48.8&515.6\\
        {PCURE-ViT14x2-k6} & 81.8&47.5&34.4&32.4&76.3&517.9\\
        {PCURE-ViT14x2-k3} & 82.1&47.8&23.9&44.1&101.8&515.8\\
         \midrule
        {PCURE-ViT14x2-k0} & 82.2&47.8&19.7&54.3&92.0&515.8 \\
        {PCURE-ViT14x1-k0} & 82.5&53.8&8.3&117.2&217.5&764.9\\
        {PCURE-ViT2x2-k0} & 82.5&61.6&2.0&433.3&1084.0&2181.6\\
         \midrule
         PatchCleanser~\cite{xiang2022patchcleanser}  &82.5&61.1&2.1&460.6&1205.0&2302.8\\
         PatchCleanser~\cite{xiang2022patchcleanser}  &82.0&55.1&12.5&82.7&216.0&646.9\\
         ECViT~\cite{ecvit}&78.6&41.7&2.25&1593.9&920.8&344.5\\
         ViP~\cite{li2022vip} (robust)&75.8&40.4&0.8&1637.2&3938.4&363.7\\
         ViP~\cite{li2022vip} (efficient)&75.3&38.3&7.7&174.2&404.4&363.7\\
         S-ViT~\cite{salman2022certified} (robust)&73.2&38.2&0.8&1577.1&920.8&344.5\\
         S-ViT~\cite{salman2022certified} (efficient)&67.3&33.0&20.5&177.3&63.3&342.5\\
         PatchGuard~\cite{xiang2021patchguard}&54.6&26.0&{162.9}&11.5&17.5&311.8\\
         BagCert~\cite{bagcert}&45.3&22.7&164.2&11.6&17.5&310.8\\
         DRS~\cite{levine2020randomized}&44.4&14.0&1.5&2476.8&920.9&258.6\\
         CBN~\cite{zhang2020clipped}&49.5&7.1&{166.3}&11.6&17.5&311.8\\
         \bottomrule
    \end{tabular}
    }
\end{table}

\section{Additional Efficiency Experiments}\label{apx-flops}

\textbf{Throughput with different batch sizes.} In Section~\ref{sec-eval}, we use a batch size of 4 for throughput evaluation. In Table~\ref{tab-efficiency-batch-size}, we report additional throughput results with different batch sizes. We can see that the throughput does not significantly change with a larger batch in our experiment setting.

\textbf{FLOP and memory analyses.} In Table~\ref{tab-flops}, we augment our evaluation results (discussed in Table~\ref{tab-undefended} and Table~\ref{tab-main}) and include latency, FLOP counts, and GPU memory consumption. We use a batch size of one for per-image inference latency. We estimate FLOPs using the \texttt{fvcore} libarary.\footnote{\url{https://github.com/facebookresearch/fvcore/blob/main/docs/flop_count.md}} For GPU memory, we report the average of the maximum allocated GPU memory for each data batch. First, we can see that throughput and latency are (inversely) correlated since we use one GPU for both two experiments. Second, the FLOP count loosely correlates to the inference throughput: models with a smaller FLOP count usually have higher inference throughput. However, theoretically, it is possible that an algorithm that has fewer FLOPs but requires sequential execution can have a much larger latency than a parallelizable algorithm with more FLOPs when running on a GPU. Third, we can see a weaker connection between GPU memory consumption and inference throughput because memory consumption heavily depends on the implementation: for example, making inferences on masked images one by one consumes less memory than batched masked images. We can see that PCURE-ViT14x2-k12 and PCURE-ViT14x2-k11 have significantly different memory footprints because we used an optimized \framework implementation for pure SRF defenses; ECViT, ViP, and S-ViT have a low inference speed but similar memory consumption as vanilla ViT due to their special implementation. 


\section{PatchCleanser and Double-masking}\label{apx-pc}

\subsection{Algorithm Details}\label{apx-pc-alg}
In this section, we discuss PatchCleanser~\cite{xiang2022patchcleanser} and its double-masking algorithm, which can be used as our secure operation $\textsc{SO}(\cdot)$ in Algorithm~\ref{alg-all}. We generally follow the original presentation of the PatchCleanser paper~\cite{xiang2022patchcleanser}, with small modifications to keep its notations consistent with this paper.

\textbf{$\mathcal{R}$-covering mask set.} The most important parameter of the double-masking algorithm is an $\mathcal{R}$-covering mask set. We copied the definition of $\mathcal{R}$-covering from the original paper~\cite{xiang2022patchcleanser}. The mask $\bfm \in \cM$ has a similar definition as the patch region set $\bfr$ discussed in Section~\ref{sec-formulation}: it is a binary tensor with the same shape as the input image $H\times W$; elements that correspond to the mask region take the value of zeros and otherwise ones. 

\begin{definition}[$\mathcal{R}$-covering]A mask set $\cM$ is $\mathcal{R}$-covering if, for any patch in the patch region set $\cR$, at least one mask from the mask set $\cM$ can cover the entire patch, i.e., 
$$\forall\ \mathbf{r}\in \mathcal{R},\ \exists\ \mathbf{m} \in \mathcal{M} \st \mathbf{m}[i,j]\leq \mathbf{r}[i,j], \  \forall (i,j)$$
\end{definition}

\begin{algorithm}[t]
    \centering
    \caption{Inference procedure of double-masking~\cite{xiang2022patchcleanser}}\label{alg-pc-inference}
    \begin{algorithmic}[1]
    \renewcommand{\algorithmicrequire}{\textbf{Input:}}
    \renewcommand{\algorithmicensure}{\textbf{Output:}}
    \Require Image or intermediate feature map $\mathbf{f}$, LRF model $\mathbb{M}_\lrf$, $\cR$-covering mask set $\mathcal{M}$
    \Ensure  Robust prediction $\bar{y}$ 
    \Procedure{SO}{$\mathbf{f},\mathbb{M}_\lrf,\mathcal{M}$}
    \State $\bar{y}_{\text{maj}},\mathcal{P}_{\text{dis}}\gets\textsc{MaskPred}(\bff,\mathbb{M}_\lrf,\mathcal{M})$
    \If{$\mathcal{P}_{\text{dis}}=\varnothing$}\label{ln-first-mask-empty-condition} \label{ln-one-mask-agreement}
    \State \Return $\bar{y}_{\text{maj}}$ \Comment{\textit{Case I: agreed prediction}}\label{ln-case-1}
    \EndIf  \label{ln-first-mask-e}
    
    \For{each $(\mathbf{m}_{\text{dis}},\bar{y}_{\text{dis}})\in\mathcal{P}_{\text{dis}}$}\label{ln-double-masking-s}\Comment{Second-rnd. mask}
       \State $\bar{y}^\prime,\mathcal{P}^\prime\gets\textsc{MaskPred}(\bff\odot\mathbf{m}_{\text{dis}},\mathbb{M}_\lrf,\mathcal{M})$ \label{ln-double-masking}
       \If{$\mathcal{P}^\prime = \varnothing$}\label{ln-two-condition-s} 
       \State \Return $\bar{y}_{\text{dis}}$ \Comment{\textit{Case II: disagreer prediction}}\label{ln-case-2}
       \EndIf\label{ln-two-condition-e}
    \EndFor\label{ln-double-masking-e}
    \State\Return $\bar{y}_{\text{maj}}$\Comment{\textit{Case III: majority prediction}}\label{ln-case-3}
    \EndProcedure
\item[]
    \Procedure{MaskPred}{$\bff,\mathbb{M}_\lrf,\mathcal{M}$}
    \State$\mathcal{P}\gets \varnothing$\Comment{A set for mask-prediction pairs}
    \For{$\mathbf{m}\in\mathcal{M}$} \Comment{Enumerate every mask $\mathbf{m}$}
    \State $\bar{y}\gets \mathbb{M}_\lrf(\bff\odot\mathbf{m})$ \Comment{Evaluate masked prediction}
    \State $\mathcal{P}\gets\mathcal{P}\ \bigcup \ \{(\mathbf{m},\bar{y})\}$ \Comment{Update set $\mathcal{P}$}
    \EndFor
    \State $\bar{y}_{\text{maj}}\gets\arg\max_{{y}^*}|\{(\mathbf{m},\bar{y})\in\mathcal{P}\ |\ \bar{y}=y^*\}|$\label{ln-get-maj-pred}\Comment{Majority}
    \State $\mathcal{P}_{\text{dis}}\gets\{(\mathbf{m},\bar{y})\in\mathcal{P} \ | \ \bar{y}\neq\bar{y}_{\text{maj}}\}$\label{ln-get-dis-set}\Comment{Disagreers}
    \State\Return $\bar{y}_{\text{maj}},\mathcal{P}_{\text{dis}}$
    \EndProcedure
    \end{algorithmic}

\end{algorithm}

\begin{algorithm}[t]
    \centering
    \caption{Certification procedure of double-masking~\cite{xiang2022patchcleanser}}\label{alg-pc-certification}
    \begin{algorithmic}[1]
    \renewcommand{\algorithmicrequire}{\textbf{Input:}}
    \renewcommand{\algorithmicensure}{\textbf{Output:}}
    \Require Image or intermediate feature map $\bff$, ground-truth label $y$, LRF model $\mathbb{M}_\lrf$, mask set $\cM$, feature-space threat model $\cA_\cR^f$
    \Ensure  Whether $\bff$ has certified robustness
    \Procedure{Cert-SO}{$\bff,\mathbb{M}_\lrf,\cM,\cA_\cR^f,y$}
    \If{$\cM$ is not $\cR$-covering}\Comment{Insecure mask set} \label{ln-check-lemma}
    \State\Return \texttt{False}
    \EndIf
    \For{every $(\mathbf{m}_0,\mathbf{m}_1) \in \mathcal{M}\times\mathcal{M}$} \label{ln-for-s}
    \State $\bar{y}^\prime \gets \mathbb{M}_\lrf(\bff\odot\mathbf{m}_0\odot\mathbf{m}_1)$\Comment{Two-mask pred.}
    \If{$\bar{y}^\prime \neq y$}
    \State\Return \texttt{False} \Comment{Possibly vulnerable}
    \EndIf
    \EndFor\label{ln-for-e}
    \State\Return \texttt{True}\Comment{Certified robustness!}
    \EndProcedure
    \end{algorithmic}
\end{algorithm}

\textbf{Mask set generation.} The PatchCleanser~\cite{xiang2022patchcleanser} discussed a systematic way to generate a mask set $\cM$ and adjust the mask set size $|\cM|$. For simplicity, we only discuss the procedure using a 1-D ``image'' example. To generate a mask set, PatchCleanser moves a mask over the input image. Formally, let us consider using a mask of width $m$ over an image of size $n$. The first mask is placed at the image coordinate of $0$ and thus covers indices from $0$ to $m-1$. Next, the mask is moved with a stride of $s$ across different image locations $\{0,s,2s,\cdots,\lfloor \frac{n-m}{s} \rfloor s\}$. Finally, the last mask is placed at the index of $n-m$ in case the mask at $\lfloor \frac{n-m}{s} \rfloor s$ cannot cover the last $m$ pixels. PatchCleanser then define a mask set $\cM_{m,s,n}$ as:
\begin{align}
    \cM_{m,s,n}=\{\mathbf{m}\in\{0,1\}^n \ | \ \mathbf{m}[u]=0, u \in [i,i+m);& \nonumber \\        
     \mathbf{m}[u]=1, u \not\in [i,i+m);&\ i\in \mathcal{I}\}\nonumber\\
     \mathcal{I}=\{0,s,2s,\cdots,\lfloor \frac{n-m}{s} \rfloor s\}\ \bigcup \ \{n-m\}\label{eqn-mask-set-generation}
\end{align}
The mask set size can be computed as:
\begin{equation}\label{eqn-mask-size}
  |\cM_{m,s,n}| = |\mathcal{I}| = \lceil\frac{n-m}{s}\rceil+1  
\end{equation}
To adjust the mask set size, PatchCleanser proposed to change the mask stride $s$. To ensure the $\cR$-covering property, PatchCleanser proved that the mask size needs to be at least $m=p+s-1$ to cover patches no larger than the size of $p$. 
In our \framework implementation, we set $s=1$ by default. We analyze the effect of stride $s$ in Appendix~\ref{apx-pc-exp}.

\textbf{Inference procedure.} We present the inference procedure of the double-masking algorithm $\textsc{SO}(\cdot)$ in Algorithm~\ref{alg-pc-inference}. Compared to the original double-masking algorithm (Algorithm 1 of \cite{xiang2022patchcleanser}), we generalize its input image $\bfx$ as input image or intermediate feature map $\bff$ and replace its vanilla model as our LRF model $\bM_\lrf$.

\textbf{Certification procedure.} We present the certification procedure $\textsc{Cert-SO}(\cdot)$ in Algorithm~\ref{alg-pc-certification}. Compared to the original double-masking algorithm (Algorithm 2 of \cite{xiang2022patchcleanser}), we generalize its input image $\bfx$ as input image or intermediate feature map $\bff$, replace its vanilla model as our LRF model $\bM_\lrf$, and consider the feature-space threat model $\cA_\cR^f$.

\begin{figure}[t]
    \centering
    \includegraphics[width=0.8\linewidth]{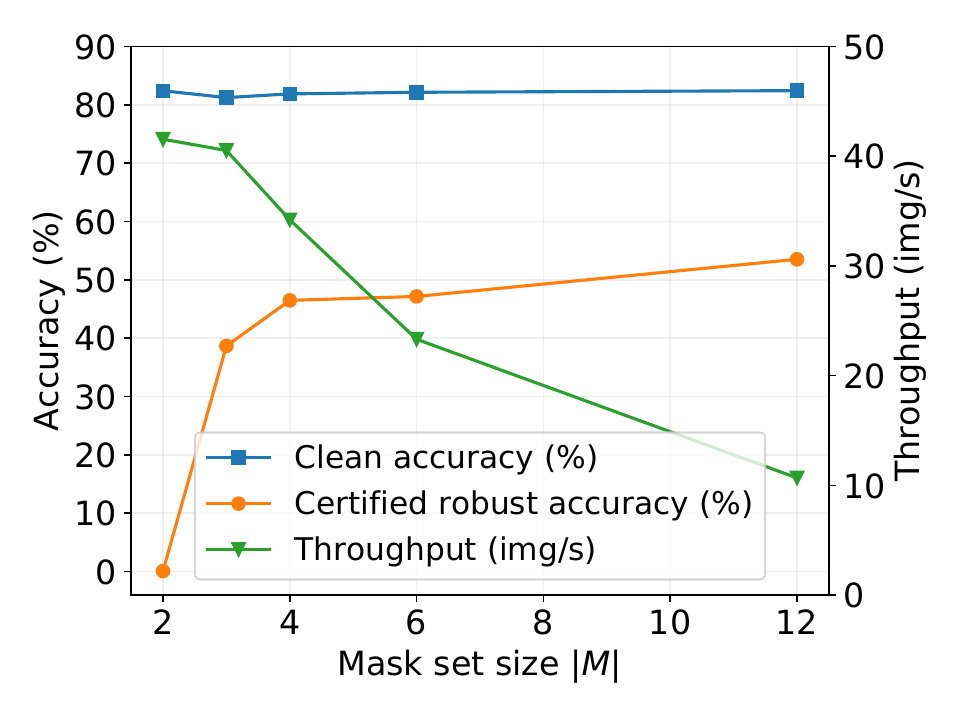}
    \caption{Effect of mask set sizes on {PCURE-ViT14x1-k3}}
    \label{fig-num-mask}
\end{figure}

\begin{table}[t]
    \centering
    \caption{Defense performance of PCURE-ViT2x2-k12 against \textit{two} 32$\times$32 patches on the 224$\times$224 ImageNet images}
    \label{tab-two-patch}
    \begin{tabular}{l|r|r|r}
    \toprule
        \multirow{2}{*}{Defense} & \multicolumn{2}{c|}{Accuracy (\%)}  &Throughput \\
          & clean & robust & (img/s)\\
         \midrule
PCURE-ViT2x2-k12&72.0&34.9&164.6\\
         \bottomrule
    \end{tabular}
\end{table}

\subsection{Additional Experimental Analyses}\label{apx-pc-exp}
In this subsection, we provide two analyses of \framework with different settings of double-masking.

\textbf{The mask stride $s$ and mask set size $|\cM|$ balances the trade-off between efficiency and robustness/utility.} In our double-masking-based \framework implementation, we set the mask stride $s=1$ by default. In Figure~\ref{fig-num-mask}, we vary the stride $s$ and the mask set size $|\cM|$ (recall $|\cM|=\lceil\frac{n-m}{s}\rceil+1$ in Equation~\ref{eqn-mask-size}) and analyze their impact on the defense performance. As shown in the figure, a smaller number of masks improves the inference throughput but hurts the clean accuracy and certified robust accuracy. This is because a smaller mask set requires to use larger masks to ensure robustness, and makes it hard for the model to make accurate predictions on masked features/images.

\textbf{\framework with double-masking works well for two adversarial patches.} In the main body of this paper, we focus on the problem of \textit{one} adversarial patch anywhere on the image. This is a challenging threat model that allows attackers to choose arbitrary patch locations and arbitrary patch content, and it is also a popular setting used in prior papers~\cite{xiang2021patchguard,levine2020randomized,zhang2020clipped,bagcert,salman2022certified,li2022vip,ecvit} for robustness comparison. In this analysis, we aim to demonstrate that \framework can also work for attackers that use multiple patches. We leverage the strategy discussed in the PatchCleanser paper~\cite{xiang2022patchcleanser}. To handle multiple ($K$) patches, PatchCleanser
generates a mask set with all possible $K$-mask combinations so that at least one of the $K$-masks can remove all patches. The robustness certification condition is that model predictions with all possible $2K$-mask combinations are all correct. In Table~\ref{tab-two-patch}, we report the performance of PCURE-ViT2x2-k12 instance for \textit{two} 32$\times$32 patches on the 224$\times$224 ImageNet images (4\% pixels in total). We can see that \framework also has non-trivial robustness against two-patch attacks.

\section{Additional Results for CIFAR-10}\label{apx-cifar}
In Section~\ref{sec-eval-detail}, we report the performance of PCURE-ViT14x2-k6 for three additional datasets (Table~\ref{tab-cifar}). In Figure~\ref{fig-cifar}, we report additional results for different \framework instances on the CIFAR-10 dataset. As shown in the figure, \framework can also effectively balance the three-way trade-off for CIFAR-10.

\begin{figure}[t]
    \centering
    \includegraphics[width=\linewidth]{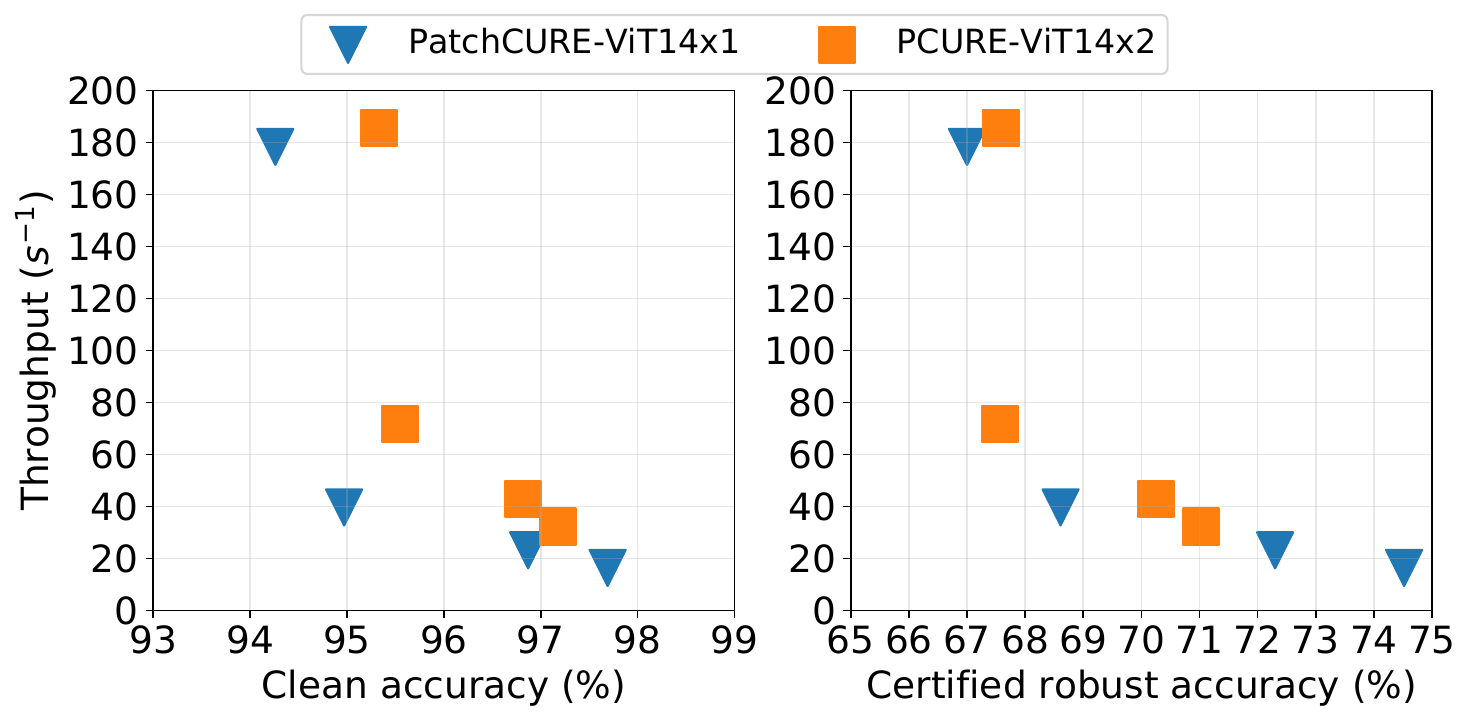}
    \caption{Performance of different \framework instances on the CIFAR-10 dataset}
    \label{fig-cifar}
\end{figure}

\begin{table}[t]
    \centering
    \caption{Performance of ViT-based defenses on CIFAR-10 (one 2\%-pixel patch on the resized 224$\times$224 images)}
    \label{tab-cifar-compare}
    \resizebox{\linewidth}{!}{
    \begin{tabular}{l|r|r|r}
    \toprule
        \multirow{2}{*}{Dataset} & \multicolumn{2}{c|}{Accuracy (\%)}  &Throughput \\
          & clean & robust & (img/s)\\
         \midrule
     PCURE-ViT14x2-k12  & 95.3    & 67.6     & 190.0     \\
     PCURE-ViT14x2-k6     & 96.8   & 70.3     & 44.4      \\
     PCURE-ViT2x2-k3      & 97.8   & 80.5       &  5.0   \\
     \midrule
         ViT-B~\cite{vit} &98.1&0&188.3\\
         PatchCleanser~\cite{xiang2022patchcleanser}/PCURE& 98.0   & 86.5  &3.8  \\
         S-ViT~\cite{salman2022certified}&90.8&67.6&0.7\\
         ECViT~\cite{ecvit}&93.5&76.4&2.3\\
         \bottomrule
    \end{tabular}}
\end{table}

In Table~\ref{tab-cifar-compare}, we report the performance of different ViT-based models on the CIFAR-10~\cite{cifar} dataset. As shown in the table, \framework has the following advantages compared to other defenses: efficient instances like PCURE-ViT14x2-k12 have significantly higher throughput and comparable robustness and utility; robust instances like PCURE-ViT-2x2-k3 match the robustness and utility performance of state-of-the-art defenses.

\section{Additional Results for PatchGuard and PatchCleanser}

In Table~\ref{tab-main}, we report the performance of PatchCleanser~\cite{xiang2022patchcleanser} and PatchGuard~\cite{xiang2021patchguard} using their default settings and models. In Table~\ref{tab-pgpc}, we report additional results for PatchCleanser~\cite{xiang2022patchcleanser} and PatchGuard~\cite{xiang2021patchguard} with the same SRF/LRF backbones used by \framework for a targeted comparison. For ViT-based backbones, \framework instances like PCURE-ViT14x2-k12 achieve significant speedup compared to PatchCleaner-ViT. Efficient \framework instances PCURE-k12 have similarly high inference throughput compared with PatchGuard, but achieve a much high certified robust accuracy. We can also see similar observations for ResNet/BagNet backbones. These comparisons demonstrate that \framework works well across different backbone models with different receptive field sizes.

\begin{table}[t]
    \centering
    \caption{Additional results for PatchCleanser~\cite{xiang2022patchcleanser} and PatchGuard~\cite{xiang2021patchguard} with different backbones}
    \label{tab-pgpc}
    \resizebox{\linewidth}{!}{
    \begin{tabular}{l|r|r|r}
    \toprule
     \multirow{2}{*}{Defenses\tnote{\dag}}  &   \multicolumn{2}{c|}{Accuracy (\%)}  &{Throughput }   \\
         &  \multicolumn{1}{c|}{clean} & \multicolumn{1}{c|}{robust} & (img/s)\\

         \midrule
        PatchCleanser-ViT-B&{82.5} (1.00)&{61.1} (1.00) &2.1 (1.0)\\ 
        {PCURE-ViT14x2-k12} & 78.3 (0.95)&44.2 (0.72)&{189.9 (90.4)}\\
        {PatchGuard-ViT14x2} & 76.9 (0.93)&23.1 (0.38)&193.4 (92.1)\\
        {PCURE-ViT14x1-k12} & 76.3 (0.92)&47.1 (0.77)&{182.0 (86.7)}\\
        {PatchGuard-ViT14x1} & 75.2 (0.91)&29.9 (0.49)&188.7 (89.8)\\
        {PCURE-ViT2x2-k12} & 71.3 (0.86)&46.8 (0.77)&{158.1 (75.3)}\\
        {PatchGuard-ViT2x2} & 69.6 (0.84)&33.0 (0.54)&164.7 (78.4)\\
        \midrule
        PatchCleanser-ResNet50 &78.3 (0.95)&53.1 (0.87)&10.6 (5.8)\\
        {PCURE-BagNet17-k50} & 65.4 (0.79)&42.2 (0.69)&{115.1 (54.8)}\\
        {PatchGuard-BagNet17} & 60.4 (0.73)&27.6 (0.45)&181.2 (86.3)\\
        {PCURE-BagNet33-k50} & 70.8 (0.85)&44.1 (0.72)&{136.8 (65.1)}\\
        {PatchGuard-BagNet33} & 67.2 (0.81)&24.0 (0.39)&183.4 (87.3)\\
        {PCURE-BagNet45-k50} & 72.4 (0.88)&34.8 (0.57)&{132.5 (63.1)}\\
        {PatchGuard-BagNet45} & 67.9 (0.82)&16.0 (0.26)&175.8 (83.7)\\

         \bottomrule
    \end{tabular}}

\end{table}

\end{document}